\documentclass[]{bytedance_seed}

\usepackage[toc,page,header]{appendix}

\usepackage{minitoc}
\usepackage{float}

\usepackage{amsmath,amsfonts,amsthm,amssymb,color}
\usepackage{mathrsfs}
\usepackage{booktabs,multirow,siunitx}

\newtheorem{theorem}{Theorem}
\newtheorem{theorem1}{Theorem}

\newtheorem{lemma}[theorem1]{Lemma}

\theoremstyle{definition}

\usepackage[procnumbered,ruled,boxed,linesnumbered]{algorithm2e}
\DontPrintSemicolon
\SetKw{KwAnd}{and}
\SetProcNameSty{textsc}
\SetFuncSty{textsc}
\SetKwRepeat{Do}{do}{while}

\def\eq#1{\begin{equation*}\begin{split}#1\end{split}\end{equation*}}

\def\eql#1#2{\begin{equation}{#1}\begin{split}#2\end{split}\end{equation}}

\def\pr#1{\left( #1 \right ) }
\def\br#1{\left[ #1 \right ] }
\def\dr#1{\left\{#1\right\}}

\def\aa{\pmb{\mathit{a}}}

\newcommand\kk{\boldsymbol{\mathit{k}}}

\newcommand\mm{\boldsymbol{\mathit{m}}}
\newcommand\oo{\boldsymbol{\mathit{o}}}
\newcommand\pp{\boldsymbol{\mathit{p}}}
\newcommand\qq{\boldsymbol{\mathit{q}}}

\renewcommand\ss{\boldsymbol{\mathit{s}}}

\newcommand\uu{\boldsymbol{\mathit{u}}}
\newcommand\vv{\boldsymbol{\mathit{v}}}

\newcommand\xx{\boldsymbol{\mathit{x}}}

\renewcommand\AA{\boldsymbol{\mathit{A}}}

\newcommand\DD{\boldsymbol{\mathit{D}}}
\newcommand\EE{\boldsymbol{\mathit{E}}}
\newcommand\FF{\boldsymbol{\mathit{F}}}
\newcommand\GG{\boldsymbol{\mathit{G}}}

\newcommand\KK{\boldsymbol{\mathit{K}}}

\newcommand\MM{\boldsymbol{\mathit{M}}}

\newcommand\OO{\boldsymbol{\mathit{O}}}
\newcommand\PP{\boldsymbol{\mathit{P}}}
\newcommand\QQ{\boldsymbol{\mathit{Q}}}

\renewcommand\SS{\boldsymbol{\mathit{S}}}

\newcommand\UU{\boldsymbol{\mathit{U}}}
\newcommand\WW{\boldsymbol{\mathit{W}}}
\newcommand\VV{\boldsymbol{\mathit{V}}}
\newcommand\XX{\boldsymbol{\mathit{X}}}
\newcommand\YY{\boldsymbol{\mathit{Y}}}

\def\tp{^\top}

\def \diag#1{\textbf{diag}\left(#1\right)}

\newcommand{\zero}{\mathbf{0}}

\newcommand \inv{^{-1}}

\newcommand \Real{\mathbb{R}}

\newcommand\dhidden{d_{\rm model}}
\newcommand\dk{d}
\newcommand\dv{d}

\newcommand\Atten{{\rm Attention}}
\newcommand\atten{\text{attention}}
\newcommand\Causal{\text{CausalAttention}}
\newcommand\causal{\text{causal-attention}}
\newcommand\Multiheadatten{\text{Multi-head-Attention}}
\newcommand\multiheadatten{\text{multi-head-attention}}
\newcommand\rsoftmax{{\rm row\_softmax}}

\newcommand\concat{\text{Concat}}
\newcommand\nhead{n_{\text{head}}}

\newcommand\x{\times}

\def\prn#1{( #1 )}
\def\prb#1{\big( #1 \big)}
\def\prbb#1{\bigg( #1 \bigg)}
\def\prB#1{\Big( #1 \Big)}

\newcommand\sigmoid{{\rm sigmoid}}
\newcommand\softmax{{\rm softmax}}

\newcommand\SSU{\SS^U}

\newcommand\de{\mathrm{d}}

\newcommand\silu{{\rm SiLU}}

\SetKwInput{KwRequire}{Require}

\def\bmn#1{$\bm{#1}$}

\title{Causal Attention with Lookahead Keys}

\author[1,2, *]{Zhuoqing Song}
\author[1]{Peng Sun}
\author[1]{Huizhuo Yuan}
\author[1, \dagger]{Quanquan Gu}

\affiliation[1]{ByteDance Seed}
\affiliation[2]{Princeton University}

\contribution[*]{Work done during internship at ByteDance Seed}
\contribution[\dagger]{Corresponding author}

\abstract{
    In standard causal attention, each token’s query, key, and value (QKV) are static and encode only preceding context. We introduce CAuSal aTtention with Lookahead kEys (CASTLE), an attention mechanism that continually updates each token’s keys as the context unfolds. We term these updated keys lookahead keys because they belong to earlier positions yet integrate information from tokens that appear later relative to those positions, while strictly preserving the autoregressive property. Although the mechanism appears sequential, we derive a mathematical equivalence that avoids explicitly materializing lookahead keys at each position and enables efficient parallel training. On language modeling benchmarks, CASTLE consistently outperforms standard causal attention across model scales, reducing validation perplexity and improving performance on a range of downstream tasks. 
}

\date{September 24, 2025}
\correspondence{Quanquan Gu at \email{quanquan.gu@bytedance.com} }

\begin{document}
\maketitle

\section{Introduction}
Causal attention~\citep{vaswani2017attention} is a cornerstone of autoregressive sequence modeling, allowing each token to condition on its past while preserving the autoregressive structure that underpins language generation.
Building on this mechanism, large language models (LLMs) have transformed natural language processing by scaling up the model size and the number of tokens trained~\citep{radford2019language,brown2020language,kaplan2020scaling}. 
While this trend has delivered impressive capabilities, it increasingly runs up against a practical bottleneck, that is high-quality tokens. 
This reality makes it imperative to improve the attention layer itself and to develop model architectures that are more token-efficient, delivering stronger performance under fixed training-token budgets. 

In standard causal attention, each token’s query, key, and value (QKV) are computed from that token’s representation and remain fixed; they cannot incorporate information from subsequent tokens. As a result, a token’s QKV can encode only its preceding context. Recent work shows that such causal mask, which blocks each token's access to its future information,
limits models' ability to capture global context, 
and impairs natural language understanding~\citep{li2023bellm,du2022glm,springerrepetition,xu2024re,zhang2025diffusion}. 
Here are some vivid illustrations that give intuitions for the limitations of causal masking.  
    Garden-path sentences~\citep{pritchett1987garden} are structurally ambiguous, often inducing an incorrect initial parse. 
    For example, ``the old man the boat''.  
    Because garden-path sentences' correct interpretation typically depends on information that appears later in the sentence, the causal mask restricting tokens to past context can cause models to struggle in resolving such ambiguities effectively~\citep{amouyal2025lm}. 
    In many tasks, the question/focus appears at the end of the input. Without access to future context, earlier tokens cannot encode the relevant information needed to anticipate the question/focus. As a result, early token representations may fail to capture important cues and global context dependencies~\citep{xu2024re}. 

To address the shortcomings of standard causal attention in pretraining, we propose a novel attention mechanism, CAuSal aTtention with Lookahead kEys (CASTLE). 
In this approach, when generating the $(t+1)$-th token, we update keys of all preceding tokens so that keys of token $s$ ($1\leq s\leq t$) are able to additionally encode information from token $s+1$ to $t$. 
These keys are called lookahead keys. 
This design preserves the autoregressive structure while allowing the keys to evolve as the context unfolds. In Figure~\ref{fig:receptivekey}, we give an illustration of receptive fields of the keys in CASTLE.
Although the mechanism appears recurrent, we establish a mathematical equivalence that avoids explicitly materializing the lookahead keys and enables efficient parallel training. 
We evaluate our approach on language modeling across multiple model scales. Experimental results show that CASTLE consistently outperforms standard causal attention in terms of validation perplexity and downstream task performance. 
We also introduce a variant, CASTLE-SWL, which applies sliding windows to lookahead keys. CASTLE-SWL has the same complexities with CASTLE in training and inference and preserves the performance gains of CASTLE while further improves efficiency in practice.  

 \begin{figure}[h!]
    \centering
    \includegraphics[width=0.82\linewidth]{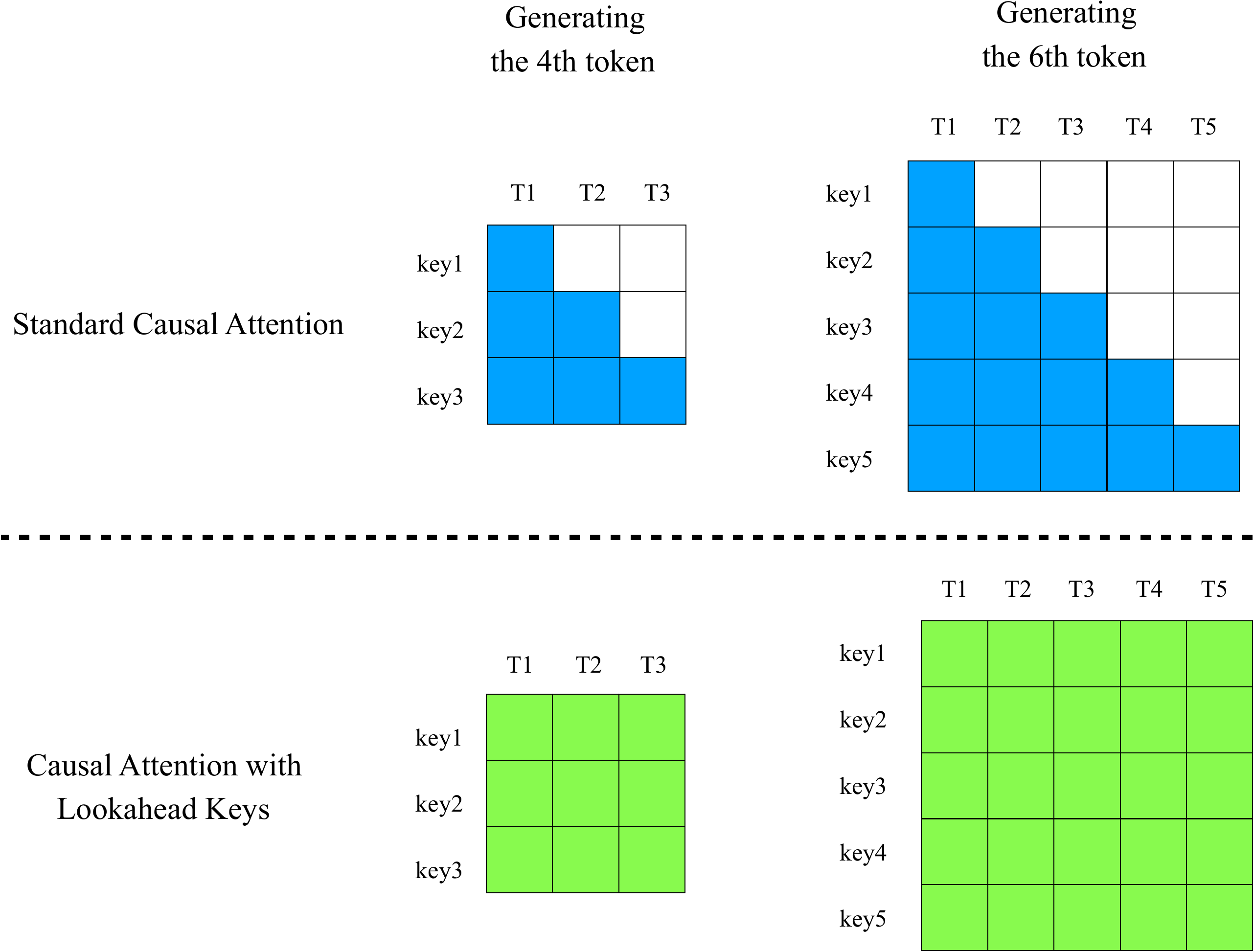}
    \caption{Receptive fields of keys in standard causal attention and CASTLE. The top row shows standard causal attention when generating the 4th token (left) and the 6th token (right). The bottom row shows CASTLE under the same settings. Here, key1, key2, $\cdots$ denote the keys corresponding to tokens 1, 2, $\cdots$, while $T_1$, $T_2$, $\cdots$ denote the tokens.
    In standard causal attention, keys are static: when generating the $(t+1)$-th token, each key $i$ with $1\leq i\leq t$ can only access information from $\dr{T_1, \cdots, T_i}$, and key $i$ remains the same for all later steps.   
    In contrast, CASTLE continuously updates keys at each prediction step, i.e., when generating the $(t+1)$-th token, the receptive field of any key $i$ with $1\leq i\leq t$, is expanded to contain information from $\dr{T_1, \cdots, T_t}$.     }
    \label{fig:receptivekey}
\end{figure} 

The remainder of this paper is organized as follows. We discuss related work in Section~\ref{sec:relwork}. In Section~\ref{sec:development1}, we elaborate on our motivations. 
In Section~\ref{sec:def}, we formally define CASTLE in its recurrent form. Section~\ref{sec:training} proves an equivalent parallel formulation of CASTLE and develops efficient parallel training algorithms. Section~\ref{sec:inference} introduces efficient inference algorithms together with the counterpart of the KV cache in CASTLE. Finally, Section~\ref{sec:experiment1} presents empirical results demonstrating the effectiveness of CASTLE across diverse tasks and model scales. 

\subsection{Related Work}\label{sec:relwork}
Several studies have observed that the causal mask, by preventing tokens from accessing future information, can hinder a model’s ability to capture global context and degrade its natural language understanding~\citep{li2023bellm,du2022glm,springerrepetition,xu2024re,zhang2025diffusion}.   

Much effort has been made to overcome this shortcoming in sentence embedding.
BeLLM (Backward Dependency Enhanced LLM)~\citep{li2023bellm} modifies decoder layers to enable sentence embeddings to integrate both past and future information. This yields substantial improvements on semantic textual similarity and downstream tasks. 
Echo Embeddings~\citep{springerrepetition} duplicate the input sequence and extract embeddings from the second occurrence, letting early tokens attend to later ones without modifying the architecture.
Similarly, Re-Reading (RE2)~\citep{xu2024re} prompts models to process inputs twice, so the second pass captures global information obtained in the first. These methods improve embedding quality and reasoning, but their benefits in large-scale pretraining remain unclear.

Selective Attention~\citep{leviathan2024selective} introduces a parameter-free modification where tokens can mask out irrelevant or outdated tokens from future attention. By subtracting accumulated selection scores from the attention logits, selective attention reduces reliance on unneeded context. As a result, it achieves substantial memory and compute savings without degrading perplexity. 
However, selective attention primarily emphasizes filtering unneeded past tokens to enhance efficiency. As discussed in the introduction, many scenarios, such as garden-path sentences or cases where the key information appears at the end of the input, require mechanisms that actively incorporate crucial future information. Whether selective attention can address this challenge remains uncertain. 

PaTH Attention~\citep{yang2025path} is a novel positional encoding scheme which introduces data-dependent encodings based on cumulative products of Householder-like transformations. Each transformation is conditioned on the input, enabling the model to dynamically adapt positional information as the sequence progresses. PaTH is related to CASTLE in the sense that in inference formulation of PaTH, keys are also updated via a rank-1 update. However, both the update mechanism of keys and the parallel training formulation in PaTH differ substantially from those in CASTLE, and PaTH remains fundamentally a positional encoding, making it orthogonal to our approach.  

Encoder-only Next Token Prediction (ENTP)~\citep{ewer2024entp} performs next-token prediction with encoder-only Transformers, where the keys are naturally re-computed at each position. It demonstrates stronger sample efficiency on small-scale language modeling and in-context learning benchmarks. However, the per-token compute in ENTP scales quadratically with sequence length and cubically over full sequences, which presents challenges for scaling.


\section{Causal Attention with Lookahead Keys}\label{sec:arbiattn}
Our motivations are discussed in Section~\ref{sec:development1}.
Formal mathematical definitions of CASTLE in recurrent form are provided in Section~\ref{sec:def}.
Direct application of the recurrent form of CASTLE in Section~\ref{sec:def} is impractical for large-scale pretraining.
To address this, we present efficient pretraining algorithms in Section~\ref{sec:training}. In Section~\ref{sec:inference}, we describe efficient inference algorithms along with the counterparts of the KV cache in CASTLE.

\subsection{Motivations}\label{sec:development1}
We first recall the standard causal attention. 
Given contextualized representations $\XX^L\in \Real^{L\x \dhidden}$ where $L$ is the sequence length and $\dhidden$ is the hidden dimension.  
The standard causal attention first computes $\QQ = \XX^L\WW_Q$, $\KK = \XX^L\WW_K$, $\VV = \XX^L\WW_V\in \Real^{L\x \dk}$, where $\dk$ is the head dimension. 
Then, the standard causal attention is computed as follows
\eql{\label{eq:standardcausal1}}{
    \Causal(\XX^L) = \rsoftmax\prB{\frac{\QQ\KK\tp}{\sqrt{\dk}} + \MM^C} \VV \in \Real^{L\x \dk},    
}
where $\MM^C \in \Real^{L\x L}$ is the causal mask which prevents each token from attending to its future tokens, i.e., $\MM^C_{ij} = 0$ if $i \geq j$ and $\MM^C_{ij} = -\infty$ otherwise. 

To explain our motivations, we begin with the recurrent form of standard causal attention. 
Consider the case when generating the $(t+1)$-th token.  
Given contextualized representations $\XX^t = 
    \begin{pmatrix}
        \xx_1 \\
        \xx_2 \\
        \vdots \\
        \xx_{t}
    \end{pmatrix} 
    \in \Real^{t\x \dhidden}$,
where the $s$-th row $\xx_s$ is the representation of token $s$. 
Unless otherwise specified, all vectors in this paper are treated as row vectors rather than column vectors.

The query, key and value of token $s$ are $\qq_s = \xx_s\WW_Q$, $\kk_s = \xx_s\WW_K$, $\vv_s = \xx_s\WW_V$. 
We also denote $\KK_t = \XX^t\WW_K$ and $\VV_t = \XX^t\WW_V$.
Then, the standard causal attention follows the recurrent form
\eql{\label{eq:standardcausal2}}{
    \causal(\XX^t) = \softmax\pr{\frac{\qq_{t}{\KK_t}\tp}{\sqrt{\dk}}} \VV_t = \frac{\sum_{s=1}^{t}\exp\prb{\qq_{t}{\kk_s}\tp / \sqrt{\dk}}\vv_s}{\sum_{s=1}^{t}\exp\prb{\qq_{t}{\kk_s}\tp / \sqrt{\dk}}} \in \Real^{1\x \dk}.       
} 
Due to the autoregressive structure, each $\xx_s$ only encodes information from token $1$ to $s$. 
Thus, when generating token $t+1$ with $t+1 > s$, each $\kk_s$ only contains information from token $1$ to $s$ without containing information from token $s+1$ to $t$. 
This can impair models' ability of natural language understanding, yielding high-quality text embedding and capturing global context as mentioned in the introduction. 

This motivates us to propose a novel attention mechanism, causal attention with lookahead keys (CASTLE), i.e., when generating the $(t+1)$-th token, we first update keys $\kk_s$ of preceding tokens $s$ with $s < t+1$ to additionally incorporate information from token $s+1$ to $t$. 
We refer to these as \emph{lookahead keys} because their representations renew with the growing context.
In this way, lookahead keys may lead to more accurate attention scores while preserving the autoregressive property.   

Before describing the details of this mechanism, we first answer the following questions.  

\noindent$\bullet$ \textbf{Why do we use lookahead keys instead of lookahead queries?} 
The answer parallels the reason why key–value (KV) pairs are cached instead of queries (Q). 
Each $\qq_s$ is only used once, namely when generating token $s+1$. 
Because we are designing an autoregressive model, past queries cannot be modified after generation, making it meaningless to update $\qq_s$. 
In contrast, $\kk_s$ is multiplied by the queries $\qq_t$ of all subsequent tokens $t \geq s$. 
Keeping $\kk_s$ updated therefore can benefit all later tokens by possibly producing more accurate attention scores. 

We also remark that updating the values $\vv_s$ similarly to lookahead keys could be beneficial while its efficient algorithm remains future study.    

\noindent$\bullet$ \textbf{How do we maintain the model autoregressive with lookahead keys?} 
When generating token $t+1$, we update keys of each preceding token $s < t+1$ with information from token $s+1$ to $t$. 
Thus, all keys only contain information from token $1$ to $t$. Queries and values are naturally only containing information from tokens up to $t$. 
No future information from tokens $k > t$ is used. 
Thus, the model maintains autoregressive property.

Further details of our design are presented in Section~\ref{sec:def}.

\subsection{Mathematical Definition in Recurrent Form}\label{sec:def}
We give mathematical definitions of CASTLE in this section. 
Let $L$ denote sequence length and $\dhidden$, $\dk$ denote the hidden dimension and head dimension, respectively.  

Throughout Section~\ref{sec:def}, we fix $t \in \dr{1, 2, \cdots, L}$ and consider the setting where $t$ tokens have been generated and the model is generating the $(t+1)$-th token. 
Denote the input contextualized representations $\XX^t = 
    \begin{pmatrix}
        \xx_1 \\
        \xx_2 \\
        \vdots \\
        \xx_{t}
    \end{pmatrix} 
    \in \Real^{t\x \dhidden}$,
where $\xx_s$ is the representation of token $s$.

Utilizing lookahead keys lies at the core of CASTLE.  
However, the way a model learns to encode information into keys of token $s$ from past tokens ($k \leq s$) may differ from how it encodes information from subsequent tokens (tokens $s < k < t+1$) when generating the $(t+1)$-th token.
To address this, we adopt a hybrid design. Specifically, we keep half of the keys the same as in standard causal which we call \emph{causal keys}, while allowing the remaining half to renew as the context progresses which we call \emph{lookahead keys}.                                

For each preceding token $s$ ($1\leq s < t+1$),
causal keys of token $s$ is a projection of $\xx_s$, while lookahead keys of token $s$ contain information from representations $\dr{\xx_{s+1}, \cdots, \xx_t}$. 
The receptive fields of causal keys and lookahead keys are illustrated in Figure~\ref{fig:keyreceptivefield}. 

\begin{figure}[ht!]
    \centering
    \includegraphics[width=0.82\linewidth]{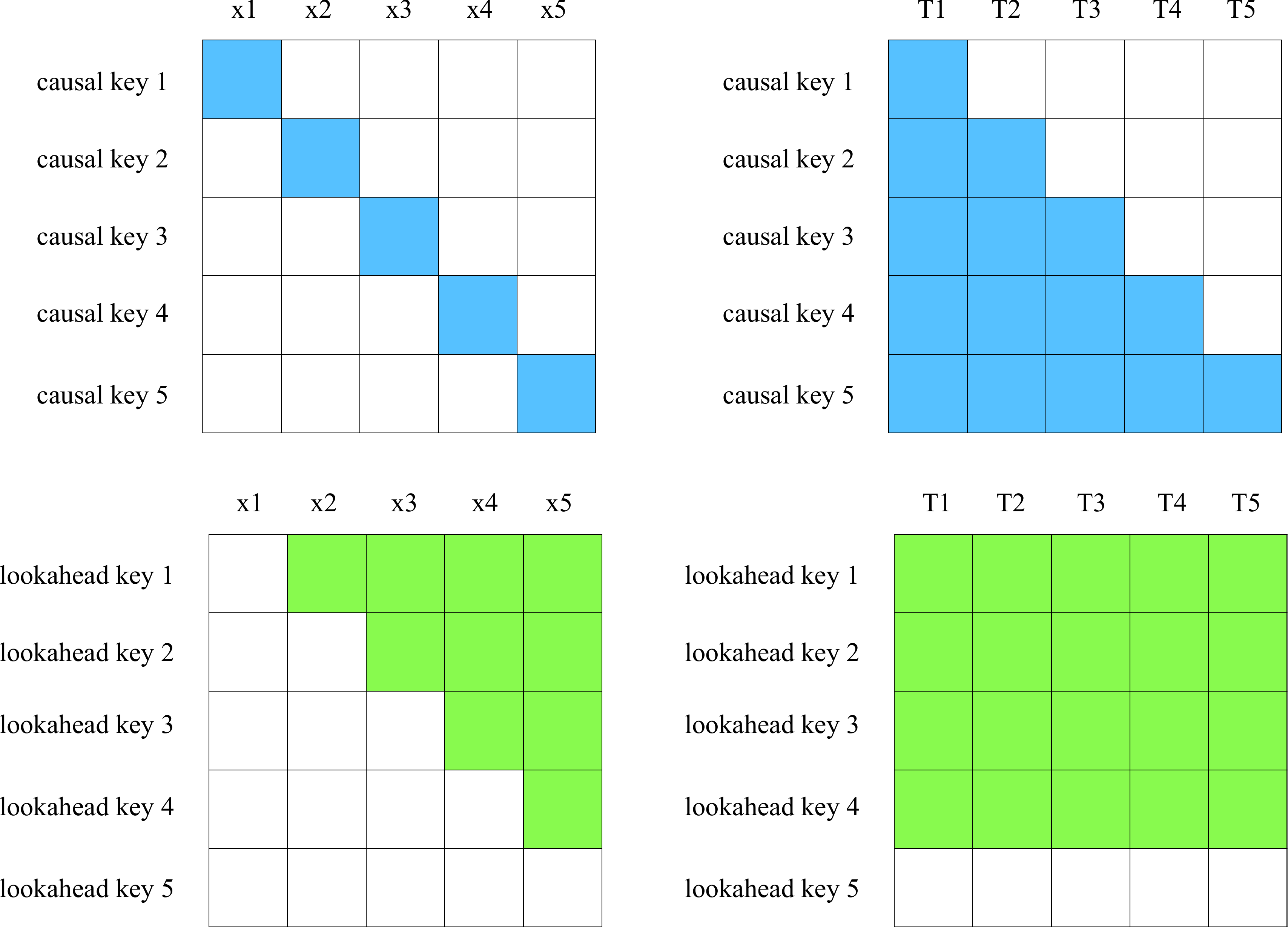}
    \caption{Receptive fields of causal keys and lookahead keys with respect to contextualized representations and tokens (excluding the first layer) when generating the 6th token. Tokens are denoted by $T_i$ and their contextualized representations by $\xx_i$. Causal key $i$ corresponds to the causal key of token $i$, while lookahead key $i$ corresponds to the lookahead key of token $i$. When generating the $(t+1)$-th token, for token $s$ ($s < t+1$), the causal key of token $s$ is a projection of $\xx_s$. Due to the softmax in attention, except in the first layer, causal keys of $s$ attend over tokens $\dr{T_1,\cdots,T_s}$. For token $s < t$, lookahead keys of $s$ incorporate information from $\dr{\xx_{s+1},\cdots,\xx_t}$ and attend over all existing tokens $\dr{T_1,\cdots,T_t}$. Since the last row of $\MM^U$ is defined as $[\MM^U_t]_{t, :} = (-\infty)^{1\x t}$, lookahead keys of token $t$ are all-zeros vectors and thus have empty receptive fields when generating the $(t+1)$-th token. }
    \label{fig:keyreceptivefield}
\end{figure}

We first project $\XX^t$ into key and value matrices $\KK^U_t$, $\VV^U_t$, $\KK^C_t$, $\VV^C_t \in \Real^{t\x \dk}$ by
\eq{
    \KK^U_t = \XX^t\WW_{K}^U \in \Real^{t\x \dhidden}, \ 
    \VV^U_t = \XX^t\WW_{V}^U \in \Real^{t\x \dhidden}, 
}
and 
\eq{
    \KK^C_t = \XX^t\WW_{K}^C \in \Real^{t\x \dhidden}, \ 
    \VV^C_t = \XX^t\WW_{V}^C \in \Real^{t\x \dhidden}    
}
as well as query matrix $\QQ^U_t = \XX^t\WW^U_Q\in \Real^{t\x \dk}$ and query vector $\qq^C_{t} = \xx_{t} \WW^C_Q \in \Real^{1\x \dk}$. 
Here, $\WW^U_Q$, $\WW^U_K$, $\WW^U_V$, $\WW^C_Q$, $\WW^C_K$, $\WW^C_V \in \Real^{\dhidden\x \dk}$ are learnable matrices.

The matrices $\KK^U_t$, $\VV^U_t$, $\QQ^U_t$ are used to generate the lookahead key $\UU^t$. Then, the causal key $\KK^C_t$ and the lookahead key $\UU^t$ are multiplied by the query vector $\qq^C_{t}$ to get the attention scores. Then, $\VV^C_t$ are multiplied by the attention weights to get the output. 
Before we elaborate on details in the definition of CASTLE, we first give formal definitions of causal keys and lookahead keys.

\noindent\textbf{Causal Keys.}
The causal keys in CASTLE are defined similarly to the keys in standard causal attention.
More specifically, causal keys are defined as 
\eq{
    \KK^C_t = 
    \begin{pmatrix}
        \kk^C_1 \\
        \kk^C_2 \\
        \vdots \\
        \kk^C_t
    \end{pmatrix} 
    = \XX^t\WW^C_K \in \Real^{t\x \dk}. 
}
The $s$-th row $\kk_s$ of $\KK^C_t$ satisfying $\kk_s = \xx_s\WW^C_K$ is the causal key of token $s$. 
Causal keys are static, i.e., the $s$-th rows of $\KK^C_t$ and $\KK^C_{t'}$ are equal to each other whenever $t, t' \geq s$.  

\noindent\textbf{Lookahead Keys.}   
We utilize a structure similar to the attention mechanism to define lookahead keys. 
An illustration for lookahead keys can be found in Figure~\ref{fig:lookaheadkey}.  

More specifically, the lookahead keys are defined as
\eql{\label{eq:Ut}}{
    \UU^t = 
    \begin{pmatrix}
      \uu^t_1 \\
      \uu^t_2 \\
      \vdots \\
      \uu^t_{t}
    \end{pmatrix}
    = \sigmoid\prb{\frac{\QQ^U_t {\KK^U_t}\tp}{\sqrt{\dk}} + \MM^U_t} \VV^U_t \in \Real^{t\x \dv},
}
where $\MM^U_t \in \Real^{t\x t}$ is a mask matrix
with 
\eql{\label{eq:MU1}}{
    [\MM^U_t]_{ij} 
    = \left\{ 
    \begin{array}{ll}
        0, & \text{if } i < j \\
        -\infty, & \text{otherwise}
    \end{array}
    \right.  
}
The $s$-th row $\uu^t_s$ of $\UU^t$ is the lookahead key of token $s$. We remark that in $\uu^t_s$, 
the superscript $t$ indicates that $\uu^t_s$ is defined when $t$ tokens have already been generated and we are about to generate the $(t+1)$-th token,    
while the subscript $s$ indicates $\uu^t_s$ is the $s$-th row of $\UU^t$. 

\begin{figure}[ht!]
    \centering
    \includegraphics[width=0.93\linewidth]{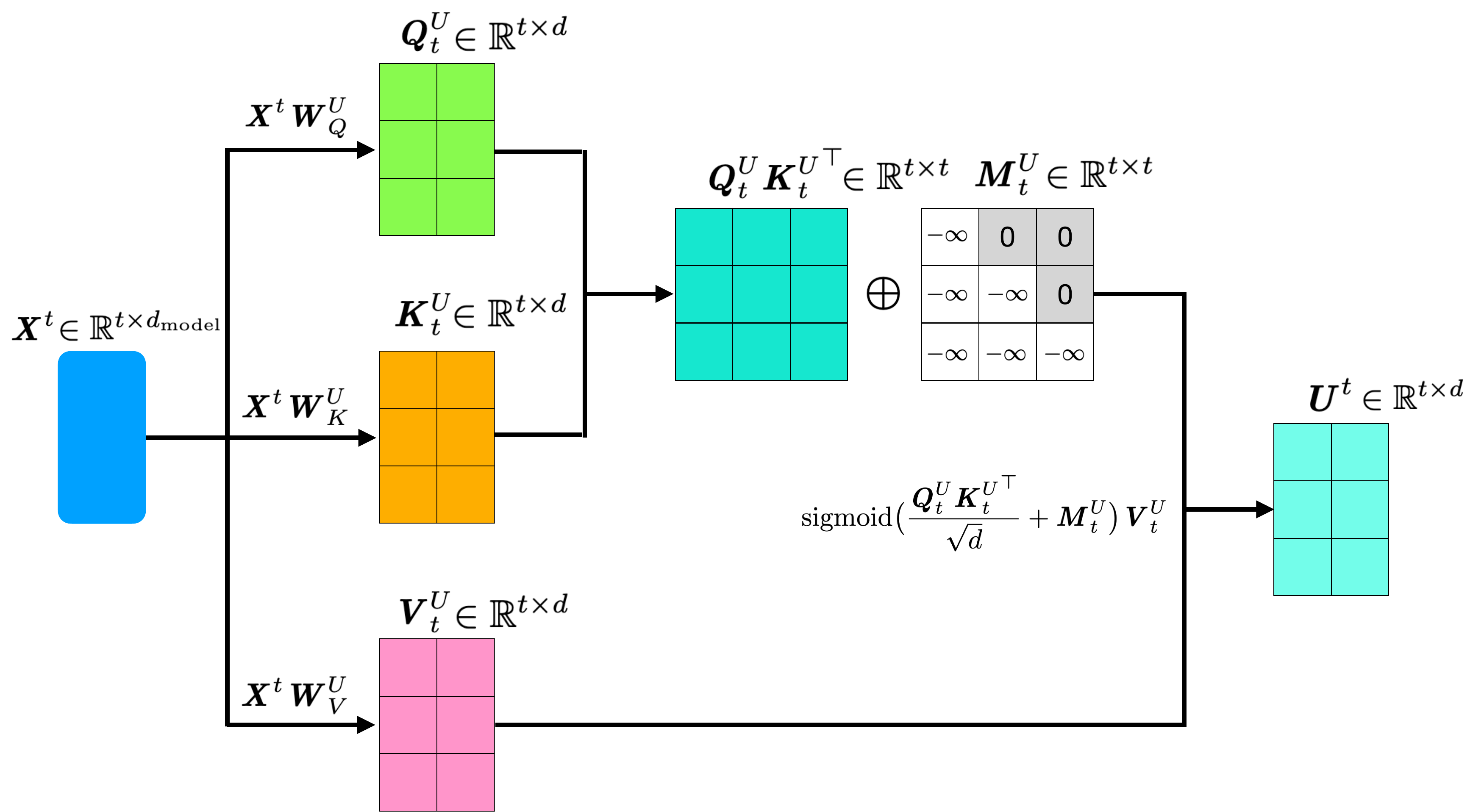}
    \caption{Illustration for the definition of lookahead keys in~\eqref{eq:Ut} when generating the $4$-th token. Let $\dhidden$ and $\dk$ denote the hidden dimension and head dimension, respectively. In this figure, we set $t=3$ and $d=2$.}
    \label{fig:lookaheadkey}
\end{figure}

The definition of $\MM^U_t$ guarantees that the lookahead key of token $s$, $\uu^t_s$, is exposed to information from $\dr{\xx_{s+1}, \cdots, \xx_t}$.   
Since $\uu^t_s$ keeps renewing as the context goes, it is natural that $\uu^t_s \neq \uu^{t+1}_s$.

We have the following remarks regarding the definition of lookahead keys in~\eqref{eq:Ut}.

\begin{itemize}
    \item 
    \textbf{Why are we using sigmoid?} The $\sigmoid$ function is used in~\eqref{eq:Ut} instead of softmax due to the consideration that when generating token $t+1$, for token $s$ with $s < t+1$, synthesizing information contained in tokens $s+1$ to $t$ should not be compulsory. 
    However, since the probabilities in softmax sum up to $1$, which forces $\uu^t_s$ to incorporate information from tokens $s+1$ to $t$ and is not desired. 
    
    \item 
    \textbf{Lookahead keys $\UU^t$ maintains autoregressive property.} First, CASTLE is defined in a recurrent form which is naturally autoregressive. Second, when generating the $(t+1)$-th token, each $\uu^t_s$ is only exposed to information from representations of tokens $s+1$ to $t$ as in~\eqref{eq:Ut}. No information from tokens which are not yet generated is exposed. 
    
    \item 
    \textbf{Lookahead keys $\UU^t$ only occur in CASTLE's recurrent form definition and inference, but cannot be materialized in parallel training.} 
    Since $\uu^t_s$ and $\uu^{t+1}_s$ may vary, 
    this prevents us from materializing $\UU^t$ for each $t$. 
    The computation cost in~\eqref{eq:Ut} is $O(t^2\dk)$. 
    If we materialize all $\UU^t$ in parallel, the computational cost is at least $\sum_{t=1}^{L} t^2 \dk = O(L^3\dk)$ which makes training on large-scale datasets impractical.
    In Section~\ref{sec:training}, we will give an equivalent form which removes the need of materializing each $\UU^t$ and enables efficient parallel training. 

\end{itemize}

\FloatBarrier
\noindent\textbf{CASTLE in Recurrent Form.} 
After defining causal keys and lookahead keys, we are ready to give the formula of CASTLE in recurrent form. 
To generate the $(t+1)$-th token, 
we utilize both the causal keys $\KK^C_t\in \Real^{t\x \dk}$  and the lookahead keys $\UU^t$. 

\begin{figure}[ht!]
    \centering
    \includegraphics[width=0.99\linewidth]{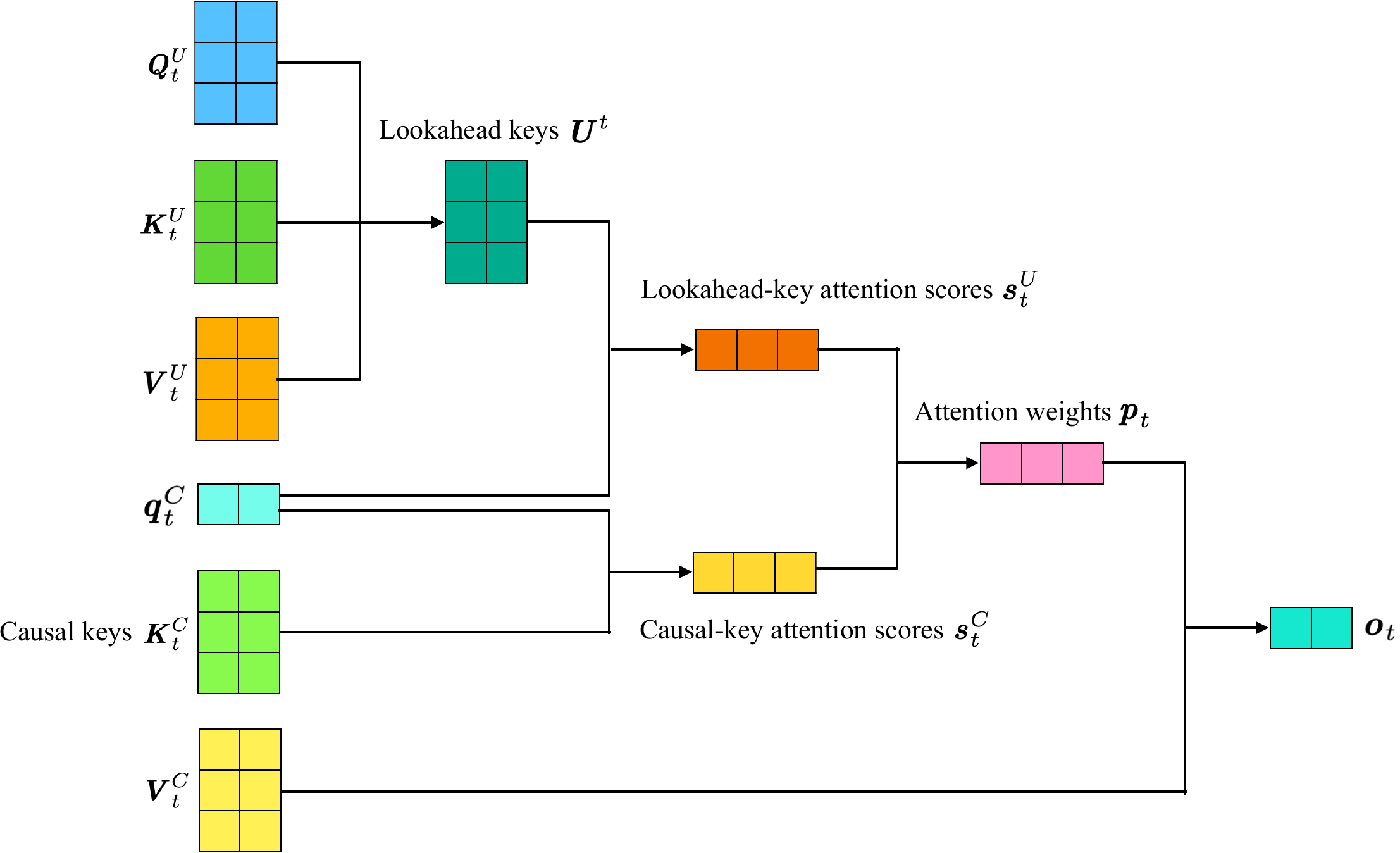}
    \caption{Illustration of CASTLE in recurrent form when generating the 4th token. 
    The causal and lookahead keys are queried by $\qq^C_t$ to generate their respective attention scores, which are combined and then goes through softmax to yield attention weights $\pp_t$. These weights are then multiplied by the value matrix $\VV^C_t$ to compute the output $\oo_t = \atten\pr{\XX^t} \in \Real^{1\x d}$ }
    \label{fig:castlerecurrent}
\end{figure}

More specifically, let the \emph{causal-key attention scores} be
\eql{\label{eq:stC}}{
    \ss^C_t = \frac{\qq^C_{t} {\KK^C_t}\tp}{\sqrt{\dk}} \in \Real^{1\x t}. 
}
Let the \emph{lookahead-key attention scores} be
\eql{\label{eq:stE}}{
    \ss^U_t = {\frac{\qq^C_{t} {\UU^t}\tp}{\sqrt{\dk}} } \in \Real^{1\x t}. 
}
Then, we define attention weights by combining the above attention scores as follows
\eql{\label{eq:ptsoftmax1}}{
    \pp_t = \softmax\pr{\ss^C_t - \silu\pr{\ss^U_t}} \in \Real^{1\x t},  
}
where $\silu(\xx) = \xx \odot \sigmoid(\xx)$.  
Then, the output is calculated as
\eql{\label{eq:headt}}{
    \atten\pr{\XX^t} = \pp_t \VV^C_t \in \Real^{1\x \dk}. 
} 
We remark that $\silu$ is applied in~\eqref{eq:ptsoftmax1} because our ablation study indicates that it plays a crucial role in ensuring training stability.  
We hypothesize that this benefit arises because many past tokens become ‘noise’ as the context grows, and the SiLU transformation effectively acts as a gate, regulating the degree to which past tokens should be down-weighted. 

An illustration of CASTLE in its recurrent form can be found in Figure~\ref{fig:castlerecurrent}.

\FloatBarrier  
\noindent\textbf{CASTLE-SWL in Recurrent Form.}
We propose a variant of CASTLE, termed CASTLE-SWL, where we apply Sliding Windows to Lookahead keys. More specifically, denote sliding window size by $W$. When generating the $(t+1)$-th token, lookahead keys of token $s$ ($s < t+1$) have access to information from $\dr{\xx_k: s+1\leq k \leq \min\dr{s+W, t}}$. Formally, the definition of lookahead keys in CASTLE-SWL follows~\eqref{eq:Ut} with $\MM^U_t$ defined as 
\eql{\label{eq:MUupperwindow}}{
    [\MM^U_t]_{ij} =
    \left\{
    \begin{array}{ll}
        0, & \text{if } i < j \leq \min\dr{t, i + W} \\
        -\infty, & \text{otherwise}
    \end{array}
    \right.  
}
An illustration for $\MM^U_t$ applied to lookahead keys in CASTLE-SWL can be found in Figure~\ref{fig:castlesupperwindow1}.   

\begin{figure}[h!]
    \centering
    \includegraphics[width=0.28\linewidth]{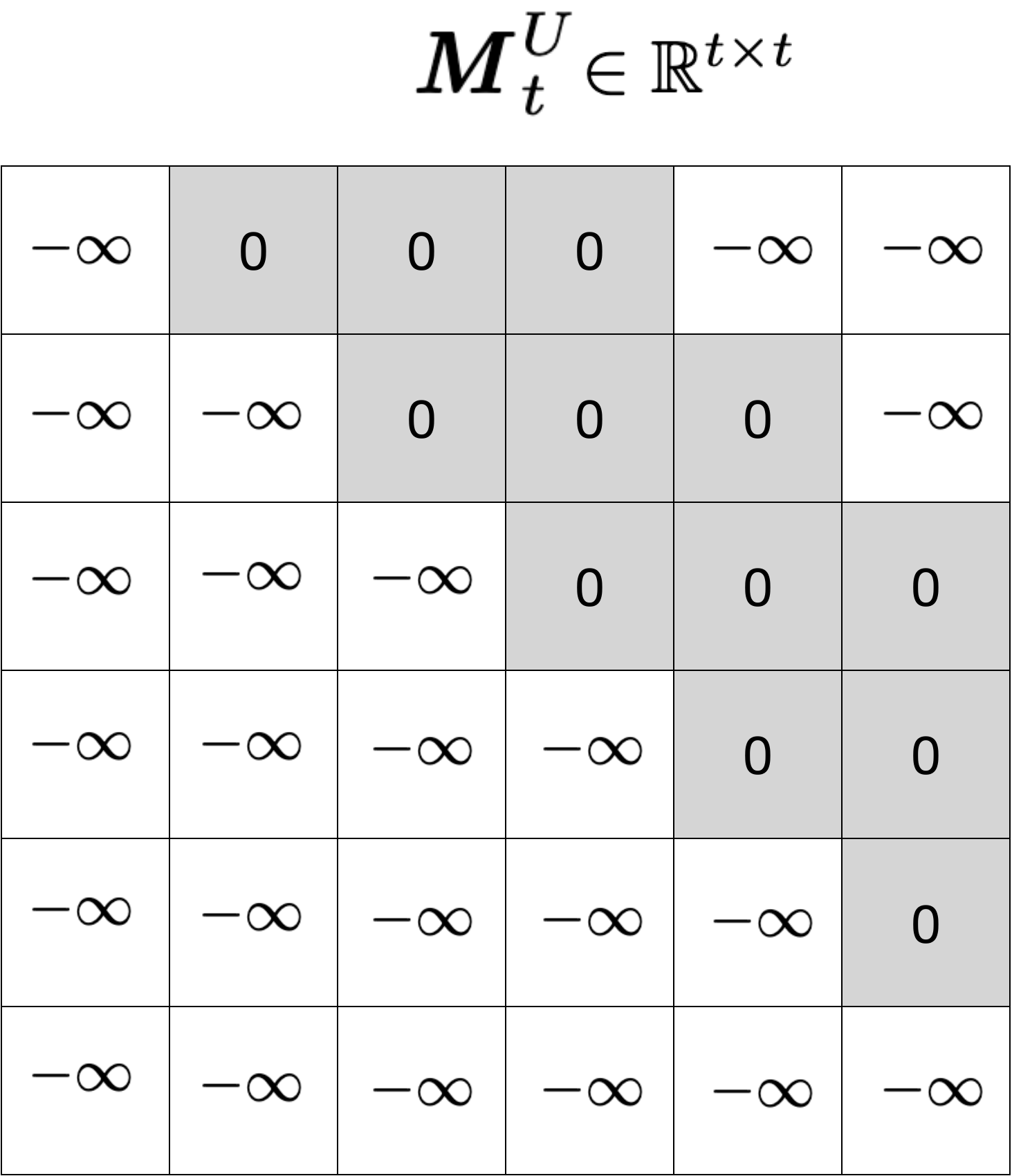}
    \caption{An illustration of $\MM^U_t$ in the lookahead keys of CASTLE-SWL, as defined in~\eqref{eq:MUupperwindow}, when generating the $(t+1)$-th token. Each lookahead key of token $s$ ($s < t$) has access to representations $\dr{\xx_k : s+1 \leq k \leq \min\dr{t,, s+W}}$, i.e., up to $W$ subsequent tokens. Also, no information from tokens which are not yet generated is leaked and thereby preserving the autoregressive property.    
    In this figure, we set $t=6$ and window size $W = 3$.  
    For example, lookahead keys of token $2$ can ``see'' $\dr{\xx_{3}, \xx_4, \xx_5}$, while lookahead keys of token $4$ can ``see'' $\dr{\xx_5, \xx_6}$. As analyzed in Section~\ref{sec:training} and~\ref{sec:inference},  
    CASTLE-SWL has the same complexities with CASTLE in both training and inference, however, it can further reduce FLOPs and improves efficiency in practice through the use of sliding windows.  }
    \label{fig:castlesupperwindow1}
\end{figure}

Sliding window modifies only lookahead keys while causal keys remains unchanged.  
Apart from changing the definition of $\MM^U_t$ from~\eqref{eq:MU1} to~\eqref{eq:MUupperwindow}, all other components of CASTLE remain the same, yielding the definition of CASTLE-SWL. 

The motivation for this design is that the semantic contribution of tokens may decay as the context unfolds. Allowing lookahead keys to aggregate information from distant tokens may introduce noise rather than useful signal. 

In addition, although CASTLE-SWL has the same complexities with CASTLE in both training and inference as analyzed in Section~\ref{sec:training} and~\ref{sec:inference}, it can further reduce FLOPs and improves efficiency in practice through the use of sliding windows.

\FloatBarrier
\subsection{Efficient Parallel Training}\label{sec:training}
In this section, we introduce our efficient parallel training algorithms.   
As discussed in Section~\ref{sec:def}, a straightforward materializing each $\UU^t$ in parallel incurs at least $O(L^3 \dk)$ computational costs. Such complexity makes training on large-scale datasets infeasible. 
To address this, we first derive an equivalent parallel formulation of CASTLE that avoids materializing lookahead keys, and then exploit the special structure of a matrix, which is expressible as a low-rank matrix multiplied by a mask, to reduce CASTLE’s training complexity to $O(L^2d)$. CASTLE-SWL naturally shares the same training complexity.    

\noindent\textbf{CASTLE in Parallel Form.}  
Let $\atten\pr{\XX^t} \in \Real^{1\x \dk}$ denote the output when generating the $(t+1)$-th token as in~\eqref{eq:headt}. 
Then, given the inputs $\XX^L \in \Real^{L\x \dhidden}$, the concatenated outputs are denoted by
\eql{\label{eq:Attention1}}{
    \Atten\prn{\XX^L} = 
    \begin{pmatrix}
        \atten\prn{\XX^1} \\
        \atten\prn{\XX^2} \\
        \vdots \\
        \atten\prn{\XX^L}
    \end{pmatrix} 
    \in \Real^{L\x \dhidden}. 
}

The following theorem provides a unified parallel formulation for both CASTLE and CASTLE-SWL that is equivalent to their recurrent forms introduced in Section~\ref{sec:def}. This formulation then serves as the basis for designing an efficient algorithm.  
Its proof is in Appendix~\ref{sec:prfthmattention1}.
\begin{theorem}\label{thm:attention1}
    Consider inputs $\XX^L\in \Real^{L\x \dhidden}$, where $L$ is the sequence length and $\dhidden$ is the hidden dimension.
    Let $\QQ^U = \XX^L \WW^U_Q$, $\KK^U = \XX^L \WW^U_K$, $\VV^U = \XX^L \WW^U_V$, $\QQ^C = \XX^L \WW^C_Q$, $\KK^C = \XX^L \WW^C_K$, $\VV^C = \XX^L \WW^C_V$.   
    Define matrix $\SSU \in \Real^{L\x L}$ as 
    \eql{\label{eq:R1}}{
        \SSU = \pr{\frac{\QQ^C {\VV^U}\tp}{\sqrt{\dk}} \odot \widetilde{\MM}^{C}} \pr{\sigmoid\pr{\frac{\QQ^U {\KK^U}\tp}{\sqrt{\dk}} + \MM^U}}\tp. 
    }
    Then, the outputs $\Atten\prn{\XX^L}$ as in~\eqref{eq:Attention1} satisfies that 
    \eql{\label{eq:Attentiontrain}}{
        \Atten\prn{\XX^L} = \rsoftmax\pr{\frac{\QQ^C {\KK^C}\tp}{\sqrt{\dk}} + \MM^{C} - \silu\prb{\SSU}} \VV^C.  
    }
    Here, $\MM^{C}$, $\widetilde{\MM}^{C}$ are the causal masks which prevent tokens from attending to their future tokens, i.e., $\MM^{C}_{ij} = 0$ if $i \geq j$ and $\MM^{C}_{ij} = -\infty$ otherwise; $\widetilde{\MM}^{C}_{ij} = 1$ if $i \geq j$ and $\widetilde{\MM}^{C}_{ij} = 0$ otherwise. 
    For CASTLE,
    $\MM^U = \MM^U_L\in \Real^{L\x L}$ with $\MM^U_L$ defined in~\eqref{eq:MU1}, while for CASTLE-SWL, $\MM^U = \MM^U_L \in \Real^{L\x L}$ with $\MM^U_L$ defined in~\eqref{eq:MUupperwindow}.    
\end{theorem}

\noindent\textbf{Efficient Parallel Training.}  
Theorem~\ref{thm:attention1} establishes the equivalence between the recurrent and parallel formulations for CASTLE and CASTLE-SWL.
However, computing $\Atten(\XX^L)$ directly from Theorem~\ref{thm:attention1} still requires $\Omega(L^3)$ complexity, since \eqref{eq:R1} involves matrix multiplications between $L$-by-$L$ matrices. 

To reduce this cost, notice that in \eqref{eq:R1}, the term $\big(\QQ^C {\VV^U}\tp\big) \odot \widetilde{\MM}^{C}$ is a masked low-rank matrix because the matrix $\QQ^C {\VV^U}\tp$ is of rank $\dk$ which is typically much smaller than $L$. 
This structure enables a more efficient computation of $\SSU$, which we exploit to design a parallel training algorithm as stated in Theorem~\ref{thm:algU1}.  
The proof of Theorem~\ref{thm:algU1} is given in Appendix~\ref{sec:prfthmalgU1}. 

\begin{theorem}\label{thm:algU1}
    Given $\XX^L$'s query, key and value matrices $\QQ^U = \XX^L \WW^U_Q$, $\KK^U = \XX^L \WW^U_K$, $\VV^U = \XX^L \WW^U_V$, $\QQ^C = \XX^L \WW^C_Q$, $\KK^C = \XX^L \WW^C_K$, $\VV^C = \XX^L \WW^C_V$, 
    for both CASTLE and CASTLE-SWL, 
    Algorithm~\ref{alg:train} (forward pass) and Algorithm~\ref{alg:traingrad} (backward pass) enable efficient parallel training and can compute $\Atten\prn{\XX^L}$ and the gradients with computational complexity $O(L^2 \dk)$ and space complexity $O(L \dk)$.  
\end{theorem}

\subsection{Efficient Inference with UQ-KV Cache}\label{sec:inference}
In this section, we introduce the inference algorithm which has unified forms for CASTLE and CASTLE-SWL. 
We first introduce the decoding algorithm.  
The decoding algorithm consists of the following 2 steps: \emph{updating step} and \emph{combining step}.  
Fix $t \in \dr{1, 2, \dots, L}$, and consider generating the $(t+1)$-th token.   

\noindent\textbf{Updating step.} 
We generate lookahead keys $\UU^t$ in the updating step. 
First, we compute $\qq^U_t = \xx_t\WW^U_Q$, $\kk^U_t = \xx_t\WW^U_K$ and $\vv^U_t = \xx_t\WW^U_V$.  
Next, rather than computing $\UU^t$ directly from~\eqref{eq:Ut}, which requires $O(t^2 d)$ computation, we update $\UU^{t}$ recursively  
\eql{\label{eq:Uupdate}}{
    \UU^t = 
    \begin{pmatrix}
        \UU^{t-1} + \sigmoid\prB{\frac{\QQ^U_{t-1}{\kk^U_{t}}\tp}{\sqrt{\dk}} + [\MM^U_t]_{:, t}} \vv^U_{t} \\
        \zero^{1\x \dk}
    \end{pmatrix},  
}
where $[\MM^U_t]_{:, t}$ is the $t$-th column of $\MM^U_t$.  
The proof of~\eqref{eq:Uupdate} is given in Appendix~\ref{sec:inferencealg}. 
Next, we discuss the caching strategy and FLOPs of the updating step.
\begin{itemize}

    \item 
    \textbf{Caching in updating step.} First, it is obvious that we need to cache $\UU^t$ so that we can implement the recursive formula~\eqref{eq:Uupdate}. Second, we need to cache $\QQ^U_t$ because $\qq^U_s$ is used in any update from $\UU^{t-1}$ to $\UU^t$ with $s \leq t-1$.  
    As $\kk^U_{t}$ and $\vv^U_{t}$ are only used in the update from $\UU^{t-1}$ to $\UU^t$ and never used again in update from $\UU^{j-1}$ to $\UU^j$ with $j \neq t$, we do not need to cache any other variables except $\UU^t$ and $\QQ^U_t$ for the updating step. 
    
    \item 
    \textbf{FLOPs in updating step.} With cached $\UU^{t-1}$ and $\QQ^U_{t-1}$, the updating formula~\eqref{eq:Uupdate} needs only $O(td)$ FLOPs. 
    And computing $\qq^U_t$, $\kk^U_t$ and $\vv^U_t$ needs $O(\dk\dhidden)$ FLOPs, yielding total FLOPs of $O(\dk\dhidden + td)$.  

\end{itemize}

\noindent\textbf{Combining step.} 
In the combining step, we compute $\qq^C_t = \xx_t\WW^C_Q$, $\kk^C_t = \xx_t\WW^C_K$ and $\vv^C_t = \xx_t\WW^C_V$.  
Next, the attention outputs are then obtained by applying~\eqref{eq:stC},~\eqref{eq:stE},~\eqref{eq:ptsoftmax1} and~\eqref{eq:headt} with $O(td)$ FLOPs. 
At this stage, since we already cached $\UU^t$ in the updating step, only $\KK^C_t$ and $\VV^C_t$ need to be stored.  

\noindent\textbf{UQ-KV cache.}
From the above analysis, the counterpart of the KV cache in CASTLE consists of $\UU^t$, $\QQ^U_t$, $\KK^C_t$, and $\VV^C_t$, which we collectively refer to as the \emph{UQ-KV cache}.
All other variables, including $\qq^C_t$, $\kk^U_t$, and $\vv^U_t$, can be safely disposed of after use.  

For the prefilling stage, let the prompt length be $L$ and inputs $\XX^L \in \Real^{L \times \dhidden}$.    
We first compute $\QQ^U_L = \XX^L\WW^U_Q$, $\KK^U_L = \XX^L\WW^U_K$, $\VV^U_L = \XX^L\WW^U_V$, $\QQ^C_L = \XX^L\WW^C_Q$, $\KK^C_L = \XX^L\WW^C_K$, $\VV^C_L = \XX^L\WW^C_V$.   
Then, we apply the forward pass of the efficient parallel training algorithm (Algorithm~\ref{alg:train}) to get $\Atten(\XX^L)$. 
For the UQ-KV cache, since we already have $\QQ^U_L$, $\KK^C_L$ and $\VV^C_L$, we only need to obtain $\UU^L$. 
This can be done similarly to FlashAttention-2~\citep{dao2023flashattention} due to the similarity between~\eqref{eq:Ut} and standard causal attention. The complete prefilling procedure is given in Algorithm~\ref{alg:prefill}.   
The analysis above leads to the following theorem.  
\begin{theorem}\label{thm:inference}
    Given inputs $\XX^L\in \Real^{L \x \dhidden}$, for both CASTLE and CASTLE-SWL, prefilling has  computational complexity $O(L\dk\dhidden + L^2\dk)$ and space complexity $O(L\dk)$,   
    and during the decoding stage, when generating the $t$-th token, the computational complexity is $O(t \dk + \dk\dhidden)$ and the UQ-KV cache requires $O(t\dk )$ memory.   
\end{theorem}

\subsection{Multi-Head Causal Attention with Lookahead Keys}
As in standard causal attention, we also utilize multi-head mechanism. 
Let $n$ denote the number of heads. 
In each head $i$, when generating the $t$-th token, denote the output as in~\eqref{eq:headt} by $\atten_i(\XX^t) \in \Real^{1\x d}$.   
Then, the output of multi-head causal attention with lookahead keys (multi-head CASTLE) is calculated as
\eql{\label{eq:ot}}{
    \multiheadatten(\XX^t) = \concat\pr{\atten_1(\XX^t), \dots, \atten_n(\XX^t)} \WW^O \in \Real^{1\x \dv},   
}
where $\WW^O \in \Real^{n\dk\x \dhidden}$ is a learnable matrix. 
The formula for forward pass in parallel training and more details of multi-head CASTLE are in Appendix~\ref{sec:prfmultiheadattention}. Multi-head CASTLE-SWL shares the same formulation and parameter count as CASTLE, and is omitted here to avoid redundancy.

\section{Experiments}\label{sec:experiment1}


\subsection{Experimental Setup}\label{sec:mainexperimentsetup}
We use the nanoGPT~\citep{karpathy2023nanogpt} code base. 
Our baseline follows the improved Transformer architecture with SwiGLU~\citep{shazeer2020glu}, Rotary Positional Embeddings~\citep{su2024roformer}, and RMSNorm~\citep{zhang2019root} following LLaMA~\citep{touvron2023llama}.  
We train models on four scales from scratch: small (0.16B), medium (0.35B),  large (0.75B) and XL (1.3B). The medium, large and XL baseline models mirror the configuration of GPT-3~\citep{brown2020language}. 
For the small setting, we increase the number of heads and the hidden dimension relative to the original GPT-3 configuration to better align parameter counts between standard causal attention and CASTLE.
CASTLE-SWL shares identical configurations with CASTLE of each model scale.    
To isolate the effect of the attention mechanism, we replace standard causal attention with CASTLE or CASTLE-SWL while keeping all other components unchanged for a fair comparison. 
We use the AdamW optimizer~\citep{loshchilov2017decoupled} and follow the training recipe of~\citep{dao2024transformers}. 
All models are trained on FineWeb-Edu dataset~\citep{lozhkovfineweb} for 50B tokens. 
Further details of experimental setup can be found in Appendix~\ref{sec:experimentsetup}. 

\subsection{Training \& Validation Loss}\label{sec:valloss}
We report training and validation loss and perplexity in Table~\ref{tab:valloss}. 
Training loss and validation loss curves are shown in Figure~\ref{fig:mhaxl}, Figure~\ref{fig:mhasmall}, Figure~\ref{fig:mhamedium} and Figure~\ref{fig:mhalarge}.  
CASTLE and CASTLE-SWL consistently outperform the baselines in both training and validation loss across all model scales. 

Specifically, after training 50B tokens, CASTLE outperforms baselines across all model scales and achieves validation losses that are 0.0059, 0.0245, 0.0356, and 0.0348 lower than the baseline for the small, medium, large, and XL models, respectively. 

We hypothesize that this improvement stems from lookahead keys in CASTLE.  
By incorporating lookahead keys, CASTLE facilitates better global context capture.
However, smaller models may struggle to fully leverage this mechanism due to limited capacity to model complex global relationships.
As a result, the improvement margin for the small model is less significant compared to larger models. 

CASTLE-SWL matches CASTLE's performance and its validation loss outperforms baselines by 0.0084, 0.0241, 0.0366, 0.0369, respectively.   
The performance gains are particularly significant in the medium, large, and XL models. 

Furthermore, as shown in Table~\ref{tab:model-configs}, both CASTLE and CASTLE-SWL have the same or fewer parameters than their baseline counterparts.  
This further underscores CASTLE and CASTLE-SWL's effectiveness in improving model performance.    

\begin{table}[h!]
\centering
\caption{Training and validation loss and perplexity for models with CASTLE, CASTLE-SWL and standard causal attention. We use ``S'', ``M'', ``L'', ``XL'' to denote model scales. Each model is trained on FineWeb-Edu for 50B tokens. The best and second best results (when showing clear improvements upon baselines) are highlighted in {bold} and {underline}, respectively.  }
\small\begin{tabular}{lccccc}
\toprule
&  &  \multicolumn{2}{c}{{Train}} & \multicolumn{2}{c}{{Eval}} \\
\cmidrule(r){3-4} \cmidrule(r){5-6}
& $n_{\text{params}}$ & {Loss} & {PPL} & {Loss} & {PPL} \\
\midrule
Baseline-S & 160M & 2.795 & 16.364 & 2.798 & 16.411 \\
{CASTLE-S} & 160M & {2.789} & {16.259} & {2.792} & {16.315} \\
{CASTLE-SWL-S} & 160M & {2.786} & {16.213} & {2.790} & {16.273} \\
\midrule 
Baseline-M & 353M & 2.641 & 14.030 & 2.639 & 14.004 \\
{CASTLE-M} & 351M & \bmn{2.616} & \bmn{13.684} & \bmn{2.615} & \bmn{13.665} \\
{CASTLE-SWL-M} & 351M & \underline{2.618} & \underline{13.708} & \underline{2.615} & \underline{13.670} \\
\midrule
Baseline-L & 756M & 2.513 & 12.346 & 2.507 & 12.269 \\
{CASTLE-L} & 753M & \underline{2.476} & \underline{11.897} & \underline{2.472} & \underline{11.840} \\ 
{CASTLE-SWL-L} & 753M & \bmn{2.476} & \bmn{11.890} & \bmn{2.471} & \bmn{11.832} \\ 
\midrule
Baseline-XL & 1.310B & 2.430 & 11.360 & 2.426 & 11.309 \\
{CASTLE-XL} & 1.304B & \bmn{2.401} & \bmn{11.031} & \underline{2.391} & \underline{10.922} \\
{CASTLE-SWL-XL} & 1.304B & \underline{2.401} & \underline{11.036} & \bmn{2.389} & \bmn{10.900} \\
\bottomrule
\end{tabular}
\label{tab:valloss}
\end{table}

\begin{figure}[htbp!]
    \centering
    \begin{subfigure}[t]{0.47\textwidth}
        \centering
        \includegraphics[width=\textwidth]{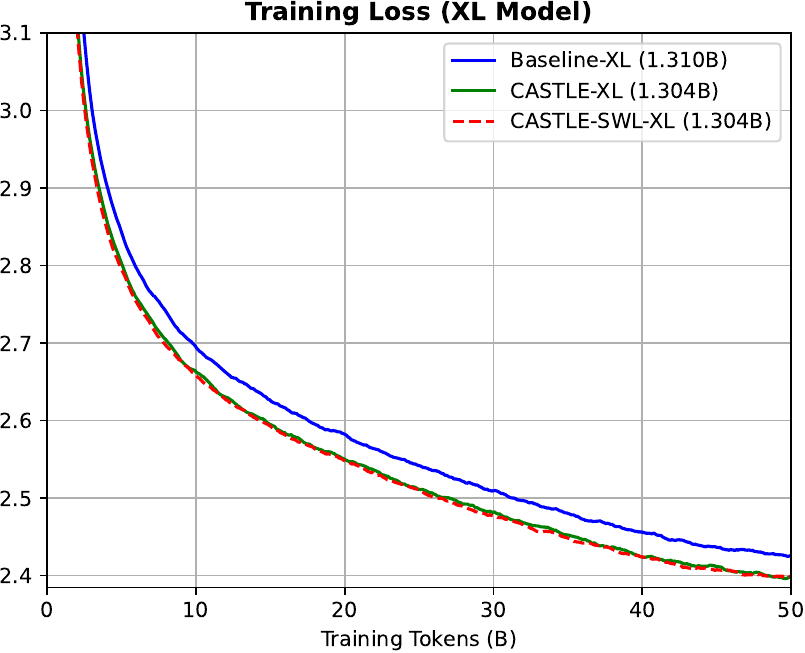}
        \label{fig:mhaxltrainingloss}
    \end{subfigure}%
    \hfill
    \begin{subfigure}[t]{0.47\textwidth}
        \centering
        \includegraphics[width=\textwidth]{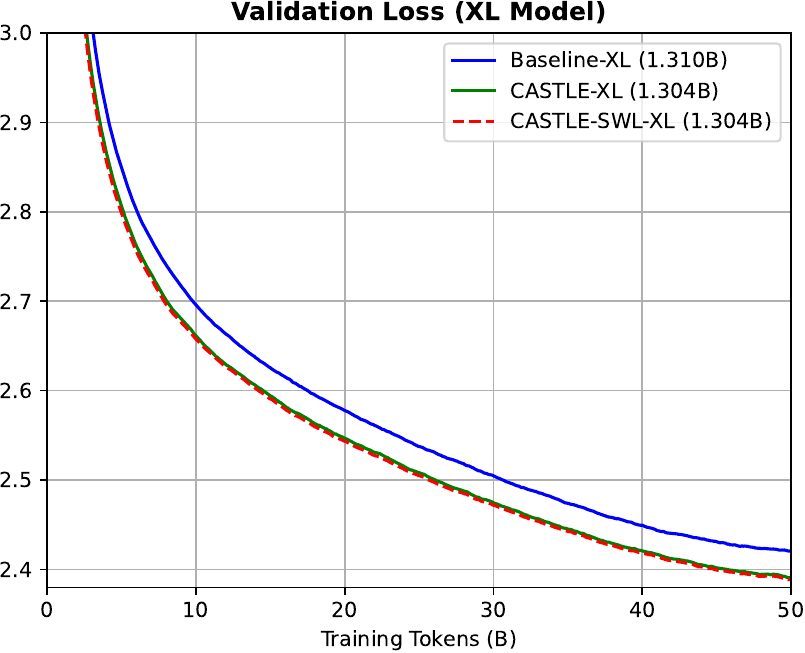}
        \label{fig:mhaxlvalloss}
    \end{subfigure}
    \caption{Training and validation loss curves of XL models. Training loss curve is smoothened with a moving window of 2000 training steps. Validation loss is evaluated every 100 training steps on 40M tokens, and its curve is smoothened by a moving window of 20 evaluation intervals. Loss curves for the small, medium and large models can be found in Figure~\ref{fig:mhasmall}, Figure~\ref{fig:mhamedium} and Figure~\ref{fig:mhalarge} of Appendix~\ref{sec:line1}. After 50B training tokens, CASTLE-XL achieves a {0.0294} lower training loss and a {0.0348} lower validation loss compared to Baseline-XL, while CASTLE-SWL-XL achieves a {0.0290} lower training loss and a {0.0369} lower validation loss compared to Baseline-XL }
    \label{fig:mhaxl}
\end{figure}

\subsection{Downstream Evaluation}\label{sec:evaltask}
We evaluate CASTLE and CASTLE-SWL on  a diverse set of downstream benchmarks, including
ARC~\citep{yadav2019quick}, BoolQ~\citep{clark2019boolq}, HellaSwag~\citep{zellers2019hellaswag}, MMLU~\citep{hendrycks2020measuring},   OBQA~\citep{mihaylov2018can}, PIQA~\citep{bisk2020piqa}, Winograde~\citep{sakaguchi2020winogrande} using lm-evaluation-harness~\citep{eval-harness}. We report normalized accuracy for ARC-Challenge, HellaSwag, OBQA, and PIQA, and standard accuracy for the other benchmarks. 
Results are reported in Table~\ref{tab:eval0shotcastle} (0-shot) and Table~\ref{tab:eval5shotcastle} (5-shot).  
Across all model scales and evaluation settings, both CASTLE and CASTLE-SWL consistently outperform the baselines in average downstream accuracy under both 0-shot and 5-shot settings. 
These findings accompany lower loss and perplexity in Section~\ref{sec:valloss} demonstrate that CASTLE and CASTLE-SWL not only lower perplexity but also translate these gains into stronger performance on diverse downstream tasks. 

\begin{table}[ht!]
\centering
\caption{Evaluation results (0-shot) for downstream tasks of different model scales. Each model is pretrained on FineWeb-Edu for 50B tokens. All values denote accuracy in percentage (\%). The higher accuracy values are shown in bold. Hella.=HellaSwag, Wino.=Winograde. }
\setlength{\tabcolsep}{4.7pt} 
\small\begin{tabular}{l|cccccccc|c}
\toprule
{Model Name} & ARC-C & ARC-E & BoolQ & Hella. & MMLU & OBQA & PIQA & Wino. & Avg. \\
\toprule
Baseline-S                 & \bmn{26.71} & \underline{54.76} & 52.51 & 35.78 & \underline{22.89} & 30.40 & 63.98 & \bmn{52.57} & 42.45 \\
{CASTLE-S}                 & 26.19 & \bmn{56.69} & \bmn{59.85} & \bmn{36.28} & \bmn{23.00} & \underline{31.60} & \underline{64.25} & \underline{52.25} & \bmn{43.76}  \\
{CASTLE-SWL-S}                 & \underline{26.62} & 54.12 & \underline{59.60} & \underline{35.97} & 22.87 & \bmn{32.60} & \bmn{64.31} & 50.51 & \underline{43.33}  \\
\midrule
Baseline-M                 & 28.58 & 60.90 & 53.61 & 43.01 & 23.21 & \underline{33.40} & \underline{67.95} & 50.91 & 45.20 \\
{CASTLE-M}                 & \bmn{30.20} & \underline{61.36} & \bmn{58.01} & \underline{43.24} & \bmn{25.34} & \bmn{34.60} & \underline{67.95} & \underline{52.64} & \bmn{46.67}   \\
{CASTLE-SWL-M}                 & \underline{29.52} & \bmn{61.41} & \underline{57.03} & \bmn{43.76} & \underline{23.29} & 33.00 & \bmn{68.12} & \bmn{53.83} & \underline{46.24}   \\
\midrule
Baseline-L                 & \underline{32.59} & \underline{65.07} & 57.49 & 47.45 & 23.57 & 32.60 & \underline{70.51} & 50.75 & 47.50 \\
{CASTLE-L}                 & 32.34 & \bmn{65.15} & \underline{57.65} & \underline{47.87} & \bmn{24.51} & \underline{35.60} & \bmn{70.78} & \underline{53.51} & \underline{48.43}  \\
{CASTLE-SWL-L}                 & \bmn{32.76} & 63.51 & \bmn{60.95} & \bmn{48.50} & \underline{23.72} & \bmn{36.00} & 70.02 & \bmn{54.85} & \bmn{48.79}  \\
\midrule
Baseline-XL                 & 33.79 & 66.62 & \underline{61.04} & 51.40 & \bmn{26.72} & 36.20 & \bmn{72.58} & 54.06 & 50.30 \\
{CASTLE-XL}                 & \underline{35.32} & \underline{67.51} & \bmn{62.81} & \bmn{52.15} & 23.74 & \underline{37.00} & 70.67 & \bmn{56.59} & \underline{50.72}  \\
{CASTLE-SWL-XL}                 & \bmn{36.43} & \bmn{69.07} & 60.24 & \underline{51.99} & \underline{24.47} & \bmn{37.40} & \underline{71.27} & \underline{55.09} & \bmn{50.74}  \\
\bottomrule
\end{tabular}
\label{tab:eval0shotcastle}
\end{table} 

\begin{table}[h!]
\centering
\caption{Evaluation results (5-shot) for downstream tasks of different model scales. The higher accuracy values are shown in bold. All values denote accuracy in percentage (\%). Each model is pretrained on FineWeb-Edu for 50B tokens. Hella.=HellaSwag, Wino.=Winograde. }
\setlength{\tabcolsep}{4.7pt} 
\small\begin{tabular}{l|cccccccc|c}
\toprule
{Model Name} & ARC-C & ARC-E & BoolQ & Hella. & MMLU & OBQA & PIQA & Wino. & Avg. \\
\toprule
Baseline-S                 & 25.68 & \underline{54.97} & \underline{56.09} & 33.81 & \bmn{25.54} & {28.20} & 63.98 & \bmn{52.57} & 42.60 \\
{CASTLE-S}                 &\underline{26.02} & 54.25 & \bmn{57.13} & \underline{35.24} & 25.22 & \underline{29.80} & \underline{64.53} & 50.99 & \underline{42.90}  \\
{CASTLE-SWL-S}                 &\bmn{27.39} & \bmn{56.19} & \underline{56.09} & \bmn{35.46} & \underline{25.34} & \bmn{30.20} & \bmn{64.85} & \underline{51.22} & \bmn{43.34}  \\
\midrule
Baseline-M                 & 31.06 & 62.46 & 48.47 & 42.83 & \underline{25.22} & \underline{33.00} & 68.39 & 51.78 & 45.40 \\
{CASTLE-M}                 & \underline{32.17} & \bmn{64.06} & \bmn{54.62} & \underline{43.47} & \underline{25.22} & \bmn{33.80} & \bmn{69.48} & \underline{52.49} & \underline{46.91}   \\
{CASTLE-SWL-M}                 & \bmn{32.85} & \underline{63.85} & \underline{54.19} & \bmn{44.18} & \bmn{26.43} & \bmn{33.80} & \underline{69.04} & \bmn{53.43} & \bmn{47.22}   \\
\midrule
Baseline-L                 & 33.36 & 63.64 & \underline{59.24} & 46.16 & \bmn{26.82} & 33.40 & \underline{69.53} & \bmn{54.06} & 48.28 \\
{CASTLE-L}                 & \bmn{37.37} & \bmn{67.89} & 50.95 & \underline{47.71} & \underline{26.11} & \underline{34.20} & \bmn{70.18} & \bmn{54.06} & \underline{48.56}  \\
{CASTLE-SWL-L}                 & \underline{36.26} & \underline{65.53} & \bmn{60.58} & \bmn{48.55} & 24.70 & \bmn{34.80} & 69.10 & \underline{53.67} & \bmn{49.15}  \\
\midrule
Baseline-XL                 & 35.58 & 65.78 & 61.07 & 50.84 & \bmn{26.71} & 36.20 & \underline{71.27} & 52.72 & 50.02 \\
{CASTLE-XL}                 & \bmn{39.08} & \bmn{70.24} & \bmn{62.60} & \underline{51.63} & 24.16 & \bmn{37.40} & 71.00 & \bmn{58.33} & \underline{51.80}  \\
{CASTLE-SWL-XL}                 & \underline{38.99} & \underline{70.08} & \underline{61.74} & \bmn{52.35} & \underline{25.85} & \underline{37.20} & \bmn{72.52} & \underline{56.75} & \bmn{51.93}  \\
\bottomrule
\end{tabular}
\label{tab:eval5shotcastle}
\end{table}

\FloatBarrier
\section{Conclusion}
We introduced CAuSal aTtention with Lookahead kEys (CASTLE), a novel attention mechanism that continually updates keys as the context evolves. This design allows each key representation to incorporate more recent information at every prediction step while strictly preserving the autoregressive property.
Although CASTLE is defined recurrently, we derived a mathematical equivalence that eliminates the need to explicitly materialize lookahead keys at each position, enabling efficient parallel training. Experimental results on language modeling demonstrate that CASTLE consistently outperforms standard causal attention, achieving lower perplexity and stronger downstream performance.


\newpage
\bibliographystyle{plainnat}
\bibliography{main}

\clearpage

\beginappendix

\section{Experimental Details and Additional Results}\label{sec:experimentsetup}

\subsection{Experimental Setup}\label{sec:experimentsetup}
We give details of our experimental setup in this section. 

\subsubsection{Model Architecture}\label{sec:baselinearchitecture}
We use the improved Transformer architecture with SwiGLU~\citep{shazeer2020glu}, Rotary Positional Embeddings~\citep{su2024roformer}, and RMSNorm~\citep{zhang2019root} following LLaMA~\citep{touvron2023llama}. 
More specifically, in each layer $l$, given the contextualized representations $\XX^{(l)} \in \Real^{L\x \dhidden}$ where $L$ is the sequence length and $\dhidden$ is the hidden dimension, then, $\XX^{(l+1)}$ is obtained by
\eq{
    \YY^{(l)} &= \ \text{MultiHead-Attention}(\text{RMSNorm}\prn{\XX^{(l)}}) \\  
    \XX^{(l+1)} &= \ \text{SwiGLU}(\text{RMSNorm}(\YY^{(l)})),   
}
where the $\text{SwiGLU}(\XX) = \prb{\text{Swish}(\XX\WW_1) \odot (\XX\WW_2)}\WW_3$. 
Here, $\WW_1\in \Real^{\dhidden\x \frac{8}{3}\dhidden}$, $\WW_2\in \Real^{\dhidden\x \frac{8}{3}\dhidden}$, $\WW_3\in \Real^{\frac{8}{3}\dhidden \x \dhidden }$
are learnable parameters. 

The function \text{MultiHead-Attention} is instantiated with either the standard multi-head causal attention or the multi-head CASTLE, or the multi-head CASTLE-SWL. 

\subsubsection{Model Configuration and Training Recipe}\label{sec:configuration}
We train models on four scales from scratch: small (0.16B), medium (0.35B),  large (0.75B) and XL (1.3B), where the medium, large and XL baseline models follow the configurations of GPT-3~\citep{brown2020language}. 
For the small setting, we increase the number of heads and the hidden dimension in the original GPT-3 configuration to better align parameter counts between standard causal attention and CASTLE. 
The configurations of the models can be found in Table~\ref{tab:model-configs}. 
To ensure a fair comparison between CASTLE and standard causal attention, we align the number of model parameters by adjusting the number of attention heads and keeping hidden dimensions and head dimensions invariant. 
This avoids changes to the representational upper bound of the models' hidden states and the behavior of RoPE, both of which depend on the hidden dimension.     
As shown in Appendix~\ref{sec:prfmultiheadattention}, the number of parameters in CASTLE matches that of standard causal attention when the number of heads satisfies the relation $n_{\text{CASTLE}} = \frac{4}{7}n_{\text{standard}}$, where $n_{\text{CASTLE}}$ and $n_{\text{standard}}$ denote the number of heads in CASTLE and standard causal attention, respectively.  
For the small model, the baseline uses 14 heads. By setting CASTLE to 8 heads, we align the parameter counts. For the other settings, the baseline models use 16 heads. As $16 * \frac{4}{7} \approx 9.14 $, we choose 9 heads for CASTLE, resulting in a slightly smaller number of parameters than the baseline. 

The model configurations and training recipes of CASTLE-SWL are identical to those of CASTLE at each model scale.   

The sliding window size for lookahead keys in CASTLE-SWL is set to 128 in CASTLE-SWL-S, and 512 in CASTLE-SWL-M, CASTLE-SWL-L, and CASTLE-SWL-XL. 
Appendix~\ref{sec:ablationwindowsize} presents an ablation study on sliding window sizes, showing that the performance of CASTLE-SWL is generally insensitive to the choice of window size within the range $[128, 1024]$.

\begin{table}[ht!]
\centering
\caption{Configurations and learning hyper-parameters (batch size in tokens and learning rate) of the models which we trained. CASTLE-SWL has exactly the same configurations and training hyper-parameters with CASTLE on each model scale, and is omitted from this table for clarity. }
\small \begin{tabular}{lcccccccc}
\toprule
{Model Name} & $n_{\text{params}}$ & $n_{\text{layers}}$ & $\dhidden$ & $n_{\text{heads}}$ & $\dk$ & \text{Batch Size} & \text{Learning Rate} \\
\toprule
Baseline-S                 & 160M   & 12 &  896 (=14 * 64)  & 14 &  64  & \multirow{2}{*}{0.5M} & \multirow{2}{*}{$6.0 \times 10^{-4}$} \\
{CASTLE-S}                 & 160M   & 12 &  896  & 8 &  64  &  &  \\
\midrule
Baseline-M                 & 353M   & 24 &  1024 (=16 * 64)  & 16 &  64  & \multirow{2}{*}{0.5M} & \multirow{2}{*}{$3.0 \times 10^{-4}$} \\
{CASTLE-M}                 & 351M   & 24 &  1024  & 9 &  64  &  &    \\
\midrule
Baseline-L                 & 756M   & 24 &  1536 (=16 * 96)  & 16 &  96  & \multirow{2}{*}{0.5M} & \multirow{2}{*}{$2.5 \times 10^{-4}$} \\
{CASTLE-L}                 & 753M   & 24 &  1536  & 9 &  96  &  &  \\
\midrule
Baseline-XL                 & 1.310B   & 24 &  2048 (=16 * 128)  & 16 &  128  & \multirow{2}{*}{0.5M} & \multirow{2}{*}{$2.0 \times 10^{-4}$} \\
{CASTLE-XL}                 & 1.304B   & 24 &  2048  & 9 &  128  &  &  \\
\bottomrule
\end{tabular}
\label{tab:model-configs}
\end{table} 

We adopt the training hyper-parameters of~\cite{dao2024transformers}.
We use the AdamW optimizer~\citep{loshchilov2017decoupled}
with $\beta_1 = 0.9$, $\beta_2 = 0.95$, weight decay rate coefficient $0.1$, and gradient clipping coefficient $1.0$.  
All models are trained with sequence length 2K and batch size 0.5M tokens.  
The small, medium, large and XL models use peak learning rates of $6.0\times 10^{-4}$, $3.0 \times 10^{-4}$, $2.5 \times 10^{-4}$ and $2.0 \times 10^{-4}$, respectively.  
All models are trained with 2000 warmup steps, and then, the cosine scheduler decays the learning rate to 10\% of the peak learning rate.  
All models are trained on the FineWeb-Edu dataset~\citep{lozhkovfineweb} for 50 billion tokens. 
The efficient parallel training algorithm of CASTLE and CASTLE-SWL (forward pass in Algorithm~\ref{alg:train} and backward pass in Algorithm~\ref{alg:traingrad}) is implemented in Triton~\citep{tillet2019triton}.

\subsection{Ablation Studies}\label{sec:abalation}
We conduct ablation studies to better understand the contributions of different design components in CASTLE and CASTLE-SWL.  
These studies address three key questions: 
\begin{itemize}
    \item Are causal keys necessary, or could lookahead keys alone suffice?
    \item Do the observed improvements arise from the mechanism of lookahead keys, or simply from increasing the number of keys?
    \item What is the role of the SiLU function in~\eqref{eq:ptsoftmax1}? 
    \item How do sliding window sizes of lookahead keys in CASTLE-SWL affect the performance?
\end{itemize}  
We systematically investigate each question in the following sections. 
All ablation experiments are trained on FineWeb-Edu for 25B tokens.  

\subsubsection{Ablations on Causal Keys}
As described in Section~\ref{sec:def}, CASTLE adopts a hybrid design that partitions keys into two groups: causal keys and lookahead keys. If all lookahead keys are replaced with causal keys, CASTLE reduces to standard causal attention. Thus, the performance gains demonstrated in Section~\ref{sec:experiment1} can be attributed to the introduction of lookahead keys. A natural follow-up question is whether causal keys are necessary, or if lookahead keys alone suffice.

To investigate this, we construct a variant of CASTLE in which all causal keys are removed. To ensure a fair comparison, we adjust the configurations so that the total parameter count remains the same, as shown in Table~\ref{tab:model-configs-causalkey}.              

\begin{table}[h!]
\centering
\caption{Configurations of CASTLE and its variant without causal keys used in ablation study on causal keys.  }
\small \begin{tabular}{lcccccc}
\toprule
{CASTLE TYPE} & $n_{\text{params}}$ & $n_{\text{layers}}$ & $\dhidden$ & $n_{\text{heads}}$ & $\dk$    \\
\toprule
{CASTLE}                 & 120M   & 12 &  768  & 6 &  64   \\
{CASTLE w/o causal keys}                 & 120M   & 12 &  768  & 7 &  64 \\    
\bottomrule
\end{tabular}
\label{tab:model-configs-causalkey}
\end{table} 

The above 2 models are both trained on FineWeb-Edu for 25B tokens for efficiency, using the same learning hyper-parameters with the small models in Section~\ref{sec:configuration}. 
Their training and validation loss and perplexity are presented in Table~\ref{tab:vallosscausalkeys}.  

\begin{table}[h!]
\centering
\caption{Training and validation loss and perplexity of CASTLE and its variant without causal keys. Each model is trained on FineWeb-Edu for 25B tokens.}
\small \begin{tabular}{lcccc}
\toprule
& \multicolumn{2}{c}{{Train}} & \multicolumn{2}{c}{{Eval}} \\
\cmidrule(r){2-3} \cmidrule(r){4-5}
& {Loss} & {PPL} & {Loss} & {PPL} \\
\midrule
{CASTLE} & \bmn{2.913} & \bmn{18.417} & \bmn{2.920} & \bmn{18.541} \\
CASTLE w/o causal keys & 3.006 & 20.213 & 3.021 & 20.505 \\  
\bottomrule
\end{tabular}
\label{tab:vallosscausalkeys}
\end{table}
As shown in Table~\ref{tab:vallosscausalkeys}, removing causal keys results in a clear degradation in performance. This demonstrates that causal keys are indispensable in CASTLE.

While these results establish the necessity of both causal and lookahead keys, our current formulation in~\eqref{eq:ptsoftmax1} employs a one-to-one pairing of a causal key and a lookahead key. An alternative design could involve grouping multiple causal keys with a single lookahead key, or vice versa. Exploring the optimal ratio between causal keys and lookahead keys is left for future work.

\subsubsection{Ablations on the Number of Keys}\label{sec:ablationnumkeys}
As discussed in Section~\ref{sec:prfmultiheadattention}, when CASTLE uses half as many heads as standard causal attention, its parameter count becomes $\tfrac{7}{8}$ of the baseline. To maintain comparable parameter counts, we adjust the number of heads accordingly. However, unlike the baseline where each head corresponds to one key, each head in CASTLE introduces two keys—one causal key and one lookahead key.

This design results in CASTLE models having 16, 18, 18, and 18 keys for the small, medium, large, and XL scales, respectively, compared to the corresponding baselines with 14, 16, 16, and 16 keys (Table~\ref{tab:model-configs}). Thus, CASTLE naturally uses slightly more keys than its baseline counterparts. A natural question arises: are the observed improvements due to the introduction of lookahead keys, or simply from having more keys overall?

To disentangle this effect, we construct CASTLE variants with only half as many heads as their baselines, ensuring that the total number of keys ($n_{\text{CausalKeys}}$ + $n_{\text{LookaheadKeys}}$) matches the baselines. This adjustment results in CASTLE having a notably smaller parameter count than the baselines. 

We also did the same ablation studies for CASTLE-SWL.

For efficiency, we train the medium and XL variants on FineWeb-Edu for 25B tokens, using the same hyper-parameters as in Section~\ref{sec:configuration}. Results are reported in Table~\ref{tab:vallosskeyalign} (CASTLE) and Table~\ref{tab:vallosskeyalign2} (CASTLE-SWL). 

As shown in Table~\ref{tab:vallosskeyalign}, despite having clearly fewer parameters, both CASTLE-M-16 and CASTLE-XL-16 outperform their baselines:
CASTLE-M-16 lags behind CASTLE-M by only 0.005 in validation loss, yet surpasses the baseline by 0.026;
CASTLE-XL-16 trails CASTLE-XL by 0.008 in validation loss, while exceeding the baseline by 0.032. 
Similar performance of CASTLE-SWL can be observed in Table~\ref{tab:vallosskeyalign2}.  

These findings confirm that CASTLE and CASTLE-SWL’s advantage stems from its mechanism of incorporating lookahead keys, rather than from increasing the number of keys from $16$ to $18$. This further consolidates the advantage of CASTLE and CASTLE-SWL.

\begin{table}[ht!]
\centering
\caption{Configurations of baseline models, CASTLE, and its variants used in the ablation study on the number of keys. Baseline-M, Baseline-XL, CASTLE-M, and CASTLE-XL follow the same configurations as in Table~\ref{tab:model-configs}. CASTLE-M-16 and CASTLE-XL-16 are constructed by reducing the number of heads in CASTLE-M and CASTLE-XL, respectively, so that the total number of keys ($n_{\text{LookaheadKeys}}$ + $n_{\text{CausalKeys}}$) matches the number of keys of the corresponding baseline models. CASTLE-SWL-M, CASTLE-SWL-XL, CASTLE-SWL-M-16, CASTLE-SWL-XL-16 have identical configurations with CASTLE-M, CASTLE-XL, CASTLE-M-16, CASTLE-XL-16, respectively and are omitted from this table for clarity.  }
\small \begin{tabular}{lcccccc}
\toprule
{Model Name} & $n_{\text{params}}$ & $n_{\text{layers}}$ & $\dhidden$ & $n_{\text{heads}}$ & $n_{\text{LookaheadKeys}}$ + $n_{\text{CausalKeys}}$ & $\dk$ \\
\toprule
Baseline-M                 & 353M   & 24 &  1024 (=16 * 64)  & 16 & 16 &  64   \\
{CASTLE-M}                 & 351M   & 24 &  1024  & 9 & 18 &  64     \\
{CASTLE-M-16}                 & 340M   & 24 &  1024  & 8 & 16 &  64     \\
\midrule
Baseline-XL                 & 1.310B   & 24 &  2048 (=16 * 128)  & 16 & 16 &  128   \\
{CASTLE-XL}                 & 1.304B   & 24 &  2048  & 9 & 18 &  128  \\
{CASTLE-XL-16}                 & 1.260B   & 24 &  2048  & 8 & 16 &  128  \\
\bottomrule
\end{tabular}
\label{tab:model-configs-keyalign}
\end{table} 

\begin{table}[h!]
\centering
\caption{Training and validation loss and perplexity of baseline models, CASTLE and CASTLE variants with the same number of keys as the baselines, after training for 25B tokens on FineWeb-Edu. The lowest loss and perplexity are shown in bold, and the second-lowest values are underlined.}
\small \begin{tabular}{lccccc}
\toprule
&  &  \multicolumn{2}{c}{{Train}} & \multicolumn{2}{c}{{Eval}} \\
\cmidrule(r){3-4} \cmidrule(r){5-6}
& $n_{\text{params}}$ & {Loss} & {PPL} & {Loss} & {PPL} \\
\midrule
Baseline-M & 353M & 2.740 & 15.483 & 2.742 & 15.523 \\
{CASTLE-M} & 351M & \bmn{2.709} & \bmn{15.018} & \bmn{2.711} & \bmn{15.039} \\
CASTLE-M-16 & 340M & \underline{2.714} & \underline{15.093} & \underline{2.716} & \underline{15.126} \\ 
\midrule
Baseline-XL & 1.310B & 2.548 & 12.779 & 2.543 & 12.723 \\
{CASTLE-XL} & 1.304B & \bmn{2.507} & \bmn{12.267} & \bmn{2.503} & \bmn{12.219} \\
CASTLE-XL-16 & 1.260B & \underline{2.514} & \underline{12.349} & \underline{2.511} & \underline{12.316} \\ 
\bottomrule
\end{tabular}
\label{tab:vallosskeyalign}
\end{table}

\begin{table}[h!]
\centering
\caption{Training and validation loss and perplexity of baseline models, CASTLE-SWL and CASTLE-SWL variants with the same number of keys as the baselines, after training for 25B tokens on FineWeb-Edu. The lowest loss and perplexity are shown in bold, and the second-lowest values are underlined.}
\small \begin{tabular}{lccccc}
\toprule
&  &  \multicolumn{2}{c}{{Train}} & \multicolumn{2}{c}{{Eval}} \\
\cmidrule(r){3-4} \cmidrule(r){5-6}
& $n_{\text{params}}$ & {Loss} & {PPL} & {Loss} & {PPL} \\
\midrule
Baseline-M & 353M & 2.740 & 15.483 & 2.742 & 15.523 \\
{CASTLE-SWL-M} & 351M & \bmn{2.710} & \bmn{15.036} & \bmn{2.713} & \bmn{15.068} \\
CASTLE-SWL-M-16 & 340M & \underline{2.716} & \underline{15.117} & \underline{2.718} & \underline{15.150} \\ 
\midrule
Baseline-XL & 1.310B & 2.548 & 12.779 & 2.543 & 12.723 \\
{CASTLE-SWL-XL} & 1.304B & \bmn{2.506} & \bmn{12.255} & \bmn{2.503} & \bmn{12.217} \\
CASTLE-SWL-XL-16 & 1.260B & \underline{2.513} & \underline{12.339} & \underline{2.508} & \underline{12.276} \\ 
\bottomrule
\end{tabular}
\label{tab:vallosskeyalign2}
\end{table}

\FloatBarrier
\subsubsection{Ablations on SiLU function in~\eqref{eq:ptsoftmax1}}
We investigate the role of the SiLU activation introduced in~\eqref{eq:ptsoftmax1}. 
When training CASTLE-XL and CASTLE-SWL-XL without $\silu$, we consistently observed the loss blows up and shows \texttt{NaN} then.
In particular, CASTLE-SWL-XL blows up around step 24300, whereas CASTLE-XL blows up earlier, around step 1500.  
Lowering the learning rate mitigates this instability. For example, CASTLE-SWL-XL without $\silu$ can remain stable up to at least 25B tokens when trained with peak learning rates of $1\times 10^{-4}$ or $5\times 10^{-5}$. 
However, such reductions in learning rate lead to noticeably worse performance, as shown in Table~\ref{tab:silu}.   

\begin{table}[h!]
\centering
\caption{Ablation study of the SiLU function in~\eqref{eq:ptsoftmax1}. 
Removing SiLU causes training instability. 
More specifically,
training CASTLE-SWL-XL without SiLU in~\eqref{eq:ptsoftmax1} under the same learning rate ($2\times 10^{-4}$) as in Table~\ref{tab:model-configs} will have the loss curve blow up at around 24300 training step.
Reducing the learning rate alleviates this instability issue but results in significant performance degradation.  }
\small \begin{tabular}{lccccc}
\toprule
&  &  \multicolumn{2}{c}{{Train}} & \multicolumn{2}{c}{{Eval}} \\
\cmidrule(r){3-4} \cmidrule(r){5-6}
& Learning Rate & {Loss} & {PPL} & {Loss} & {PPL} \\
\midrule
{CASTLE-SWL-XL} & $2\times 10^{-4}$ & {2.506} & {12.255} & {2.503} & {12.217} \\
{CASTLE-SWL-XL w/o SiLU} & $1\times 10^{-4}$ & {2.523} & {12.468} & {2.520} & {12.424}  \\
{CASTLE-SWL-XL w/o SiLU} & $5\times 10^{-5}$ & {2.571} & {13.084} & {2.571} & {13.075} \\
\bottomrule
\end{tabular}
\label{tab:silu}
\end{table}

\FloatBarrier
\subsubsection{Ablations on Sliding Window Size in Lookahead Keys of CASTLE-SWL}\label{sec:ablationwindowsize}
We investigate the effect of sliding window size on CASTLE-SWL across different model scales. 
For efficiency, each model is trained on the FineWeb-Edu dataset for 25B tokens. 
We evaluate CASTLE-SWL with window sizes of 128, 256, 512, and 1024, and report both training and validation loss and perplexity. 
The results for small, medium, large, and XL models are shown in Table~\ref{tab:windowsizeS}, Table~\ref{tab:windowsizeM}, Table~\ref{tab:windowsizeL}, and Table~\ref{tab:windowsizeXL}, respectively. 
Throughout, we use the number following ``SWL'' to denote the sliding window size, e.g., CASTLE-SWL128-S refers to CASTLE-SWL-S with a window size of 128. 

As shown in Table~\ref{tab:windowsizeS}, a window size of 128 achieves the best performance for the small model. 
For medium, large, and XL models (Tables~\ref{tab:windowsizeM}, \ref{tab:windowsizeL}, and \ref{tab:windowsizeXL}), the optimal performance is obtained with a window size of 512. 
Based on these findings, we adopt window sizes of 128, 512, 512, and 512 for small, medium, large, and XL models, respectively, in the experiments presented in Section~\ref{sec:experiment1}.

Overall, CASTLE-SWL performance is not sensitive to the choice of sliding window size. 
For example, for the medium model, the best sliding window size (512) reduces validation loss by $0.0297$ compared to the baseline, but the gap between the best (512) and worst window sizes (256) is only $0.0038$. 
Similarly, for the XL model, the best sliding window (512) improves upon the baseline by $0.0406$ while the difference between the best (512) and worst (128) sliding window sizes is just $0.0099$.   
These results suggest that while CASTLE-SWL consistently improves over the baselines across all model scales and sliding window sizes tested, and its performance is relatively robust to the choices of window sizes.

\begin{table}[h!]
\centering
\caption{Ablations on sliding window sizes for CASTLE-SWL-S. Training and validation loss and perplexity of baseline models, CASTLE-SWL-S with different sliding window sizes after training for 25B tokens on FineWeb-Edu are reported. The lowest loss and perplexity are shown in bold, and the second-lowest values are underlined.}
\small \begin{tabular}{lccccc}
\toprule
&  &  \multicolumn{2}{c}{{Train}} & \multicolumn{2}{c}{{Eval}} \\
\cmidrule(r){3-4} \cmidrule(r){5-6}
& $n_{\text{params}}$ & {Loss} & {PPL} & {Loss} & {PPL} \\
\midrule
Baseline-S & 160M & 2.892 & 18.037 & 2.901 & 18.197 \\
CASTLE-SWL128-S & 160M & \bmn{2.883} & \bmn{17.859} & \bmn{2.889} & \bmn{17.971} \\ 
CASTLE-SWL256-S & 160M & 2.888 & 17.952 & 2.895 & 18.092 \\ 
CASTLE-SWL512-S & 160M & 2.885 & 17.910 & \underline{2.892} & \underline{18.023} \\ 
CASTLE-SWL1024-S & 160M & \underline{2.885} & \underline{17.901} & 2.892 & 18.031 \\ 
\bottomrule
\end{tabular}
\label{tab:windowsizeS}
\end{table}

\begin{table}[h!]
\centering
\caption{Ablations on sliding window sizes for CASTLE-SWL-M. Training and validation loss and perplexity of baseline models, CASTLE-SWL-M with different sliding window sizes after training for 25B tokens on FineWeb-Edu are reported. The lowest loss and perplexity are shown in {bold}, and the second-lowest values are {underlined}.}
\small \begin{tabular}{lccccc}
\toprule
&  &  \multicolumn{2}{c}{{Train}} & \multicolumn{2}{c}{{Eval}} \\
\cmidrule(r){3-4} \cmidrule(r){5-6}
& $n_{\text{params}}$ & {Loss} & {PPL} & {Loss} & {PPL} \\
\midrule
Baseline-M & 353M & 2.740 & 15.483 & 2.742 & 15.523 \\
CASTLE-SWL128-M & 351M & 2.715 & 15.098 & 2.716 & 15.124 \\ 
CASTLE-SWL256-M & 351M & 2.715 & 15.112 & 2.716 & 15.127 \\ 
CASTLE-SWL512-M & 351M & \bmn{2.710} & \bmn{15.036} & \bmn{2.713} & \bmn{15.068} \\ 
CASTLE-SWL1024-M & 351M & \underline{2.713} & \underline{15.071} & \underline{2.715} & \underline{15.103} \\ 
\bottomrule
\end{tabular}
\label{tab:windowsizeM}
\end{table}

\begin{table}[h!]
\centering
\caption{Ablations on sliding window sizes for CASTLE-SWL-L. Training and validation loss and perplexity of baseline models, CASTLE-SWL-L with different sliding window sizes after training for 25B tokens on FineWeb-Edu are reported. The lowest loss and perplexity are shown in bold, and the second-lowest values are underlined.}
\small \begin{tabular}{lccccc}
\toprule
&  &  \multicolumn{2}{c}{{Train}} & \multicolumn{2}{c}{{Eval}} \\
\cmidrule(r){3-4} \cmidrule(r){5-6}
& $n_{\text{params}}$ & {Loss} & {PPL} & {Loss} & {PPL} \\
\midrule
Baseline-L & 756M & 2.740 & 15.483 & 2.742 & 15.523 \\
CASTLE-SWL128-L & 753M & 2.597 & 13.425 & 2.596 & 13.411 \\ 
CASTLE-SWL256-L & 753M & 2.589 & 13.314 & 2.587 & 13.290 \\ 
CASTLE-SWL512-L & 753M & \bmn{2.582} & \bmn{13.219} & \bmn{2.580} & \bmn{13.202} \\ 
CASTLE-SWL1024-L & 753M & \underline{2.582} & \underline{13.229} & \underline{2.581} & \underline{13.209} \\ 
\bottomrule
\end{tabular}
\label{tab:windowsizeL}
\end{table}

\begin{table}[h!]
\centering
\caption{Ablations on sliding window sizes for CASTLE-SWL-XL. Training and validation loss and perplexity of baseline models, CASTLE-SWL-XL with different sliding window sizes after training for 25B tokens on FineWeb-Edu are reported. The lowest loss and perplexity are shown in bold, and the second-lowest values are underlined.}
\small \begin{tabular}{lccccc}
\toprule
&  &  \multicolumn{2}{c}{{Train}} & \multicolumn{2}{c}{{Eval}} \\
\cmidrule(r){3-4} \cmidrule(r){5-6}
& $n_{\text{params}}$ & {Loss} & {PPL} & {Loss} & {PPL} \\
\midrule
Baseline-XL & 1.310B & 2.548 & 12.779 & 2.543 & 12.723 \\
CASTLE-SWL128-XL & 1.304B & 2.517 & 12.393 & 2.513 & 12.339 \\ 
CASTLE-SWL256-XL & 1.304B & 2.514 & 12.353 & 2.510 & 12.300 \\ 
CASTLE-SWL512-XL & 1.304B & \bmn{2.506} & \bmn{12.255} & \bmn{2.503} & \bmn{12.217} \\ 
CASTLE-SWL1024-XL & 1.304B & \underline{2.514} & \underline{12.351} & \underline{2.509} & \underline{12.294} \\ 
\bottomrule
\end{tabular}
\label{tab:windowsizeXL}
\end{table}

\FloatBarrier
\subsection{Additional Loss Curves}\label{sec:line1}
This section presents additional training and validation loss curves for the small, medium, and large models, while the loss curves for the XL models are already shown in Figure~\ref{fig:mhaxl}. Each figure compares the baseline with CASTLE and CASTLE-SWL.    
\begin{figure}[ht!]
    \centering
    \begin{subfigure}[t]{0.47\textwidth}
        \centering
        \includegraphics[width=\textwidth]{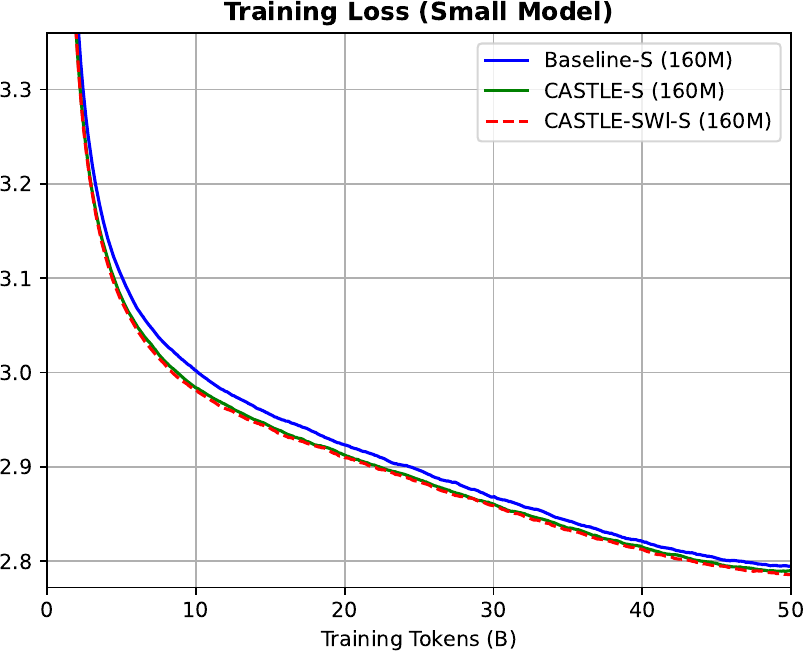}
        \label{fig:mhasmalltrainingloss}
    \end{subfigure}  
    \hfill
    \begin{subfigure}[t]{0.47\textwidth}
        \centering
        \includegraphics[width=\textwidth]{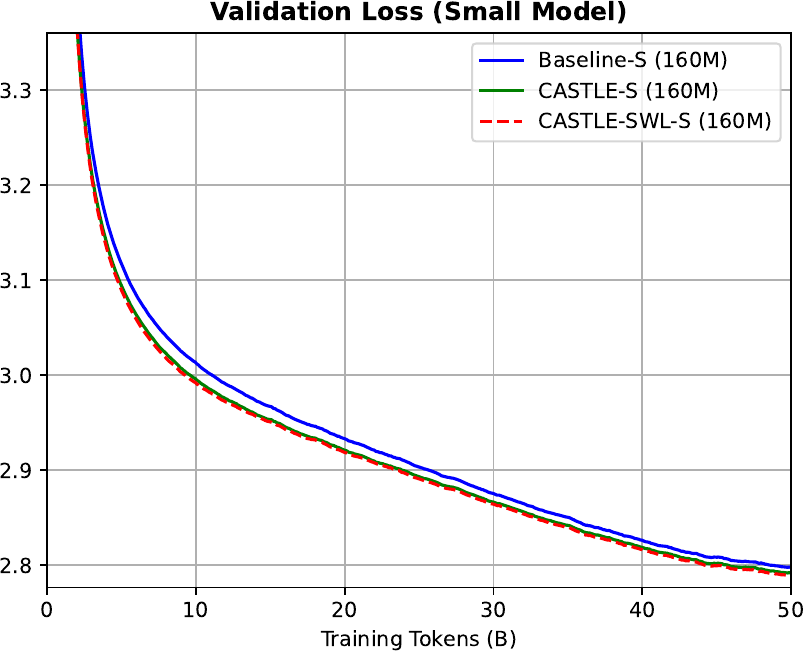}
        \label{fig:mhasmallvalloss}
    \end{subfigure}
    \caption{Training and validation loss curves of small models. Training loss curve is smoothened with a moving window of 2000 training steps. Validation loss is evaluated every 100 training steps on 40M tokens, and its curve is smoothened by a moving window of 20 evaluation intervals. As seen in Table~\ref{tab:valloss} and in comparison with Figure~\ref{fig:mhamedium}, Figure~\ref{fig:mhalarge} and Figure~\ref{fig:mhaxl}, CASTLE yields only marginal improvements over the baseline on small models. A likely explanation is that this is because the benefit of lookahead keys may lie in helping models capture global dependencies, but small models are capacity-limited and can primarily extract local features, making global relations less useful at this scale.  }
    \label{fig:mhasmall}
\end{figure}

\begin{figure}[ht!]
    \centering
    \begin{subfigure}[t]{0.47\textwidth}
        \centering
        \includegraphics[width=\textwidth]{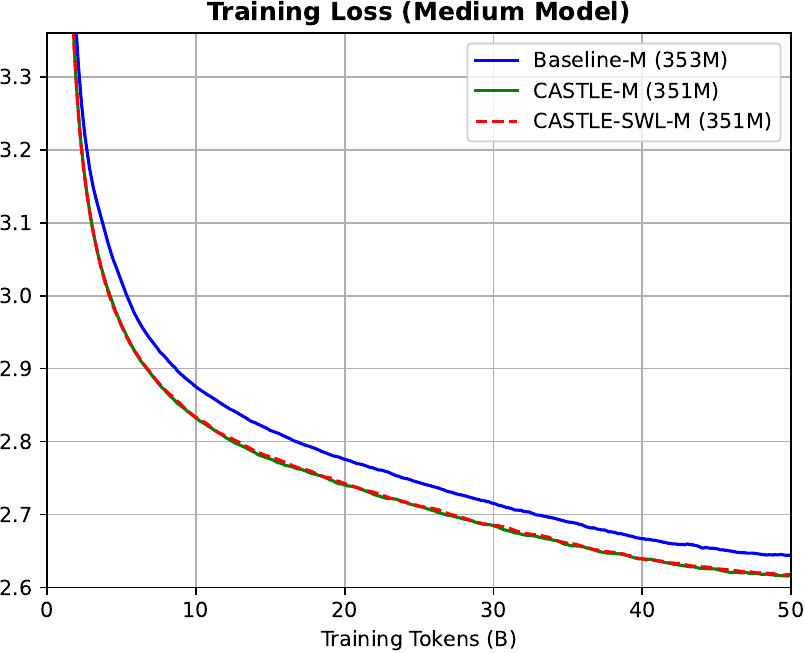}
        \label{fig:mhamediumtrainingloss}
    \end{subfigure}
    \hfill
    \begin{subfigure}[t]{0.47\textwidth}
        \centering
        \includegraphics[width=\textwidth]{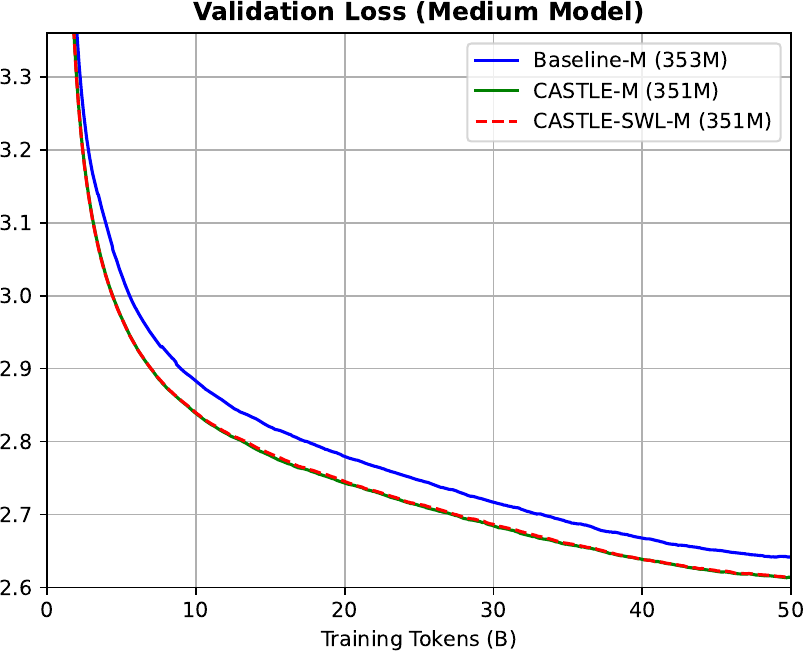}
        \label{fig:mhamediumvalloss}
    \end{subfigure}
    \caption{Training and validation loss curves of medium models. Training loss curve is smoothened with a moving window of 2000 training steps. Validation loss is evaluated every 100 training steps on 40M tokens, and its curve is smoothened by a moving window of 20 evaluation intervals. After 50B training tokens, CASTLE-M achieves a {0.0294} lower training loss and a {0.0245} lower validation loss compared to Baseline-M, while CASTLE-SWL-M achieves a {0.0232} lower training loss and a {0.0241} lower validation loss compared to Baseline-M}
    \label{fig:mhamedium}
\end{figure}

\begin{figure}[!h]
    \centering
    \begin{subfigure}[t]{0.47\textwidth}
        \centering
        \includegraphics[width=\textwidth]{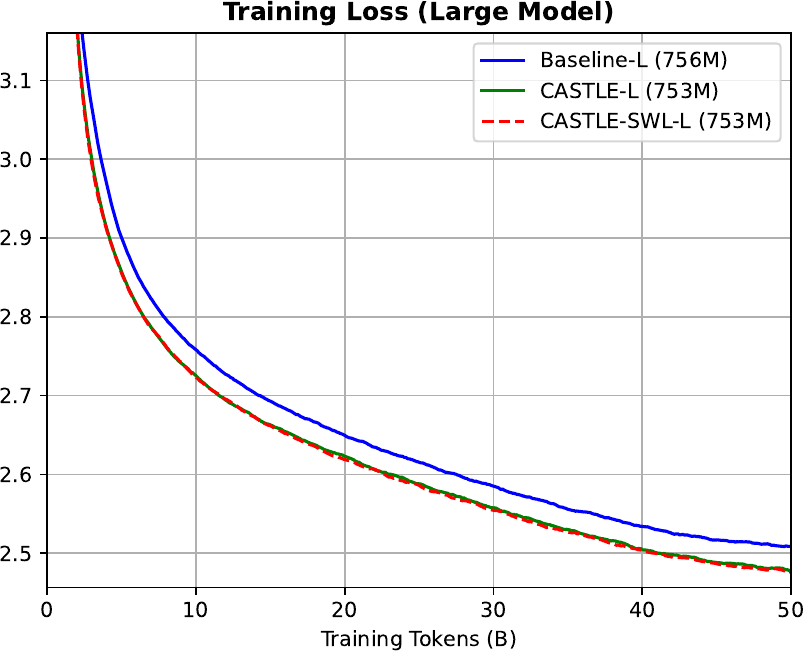}
        \label{fig:mhalargetrainingloss}
    \end{subfigure}
    \hfill
    \begin{subfigure}[t]{0.47\textwidth}
        \centering
        \includegraphics[width=\textwidth]{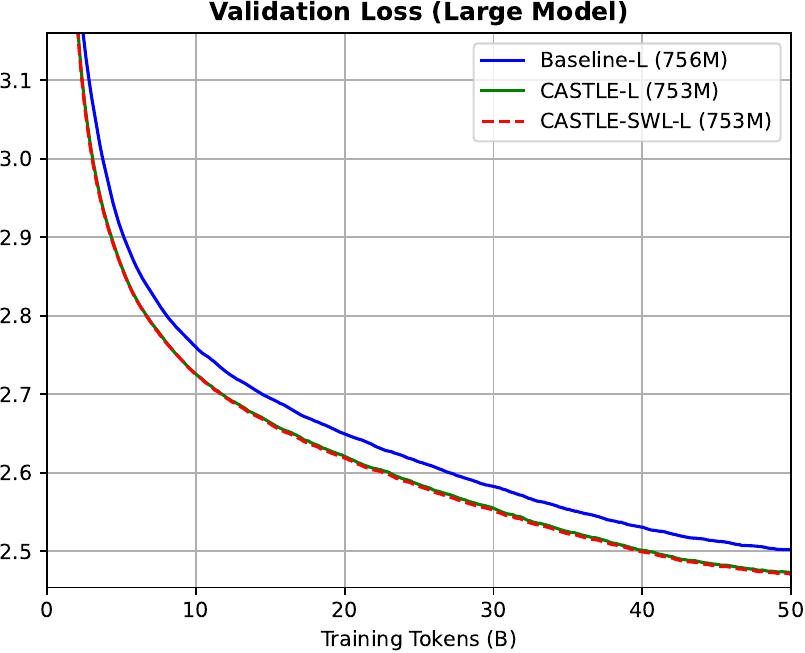}
        \label{fig:mhalargevalloss}
    \end{subfigure}
    \caption{Training and validation loss curves of large models. Training loss curve is smoothened with a moving window of 2000 training steps. Validation loss is evaluated every 100 training steps on 40M tokens, and its curve is smoothened by a moving window of 20 evaluation intervals. After 50B training tokens, CASTLE-L achieves a {0.0371} lower training loss and a {0.0356} lower validation loss compared to Baseline-L, while CASTLE-SWL-L achieves a {0.0376} lower training loss and a {0.0366} lower validation loss compared to Baseline-L }
    \label{fig:mhalarge}
\end{figure}

\FloatBarrier
\section{Proof of Theorem~\ref{thm:attention1}}\label{sec:prfthmattention1}  
First, recall the notations in Section~\ref{sec:training}. 
More specifically,  
consider inputs 
$\XX^L = 
\begin{pmatrix}
    \xx_1 \\
    \xx_2 \\
    \vdots \\
    \xx_{L}
\end{pmatrix} \in \Real^{L\x \dhidden}$, where $\xx_t$ is the representation of the $t$-th input token, $L$ is the sequence length and $\dhidden$ is the hidden dimension.  

For each $1\leq t\leq L$, denote
\eq{
    \qq^U_t &= \xx_t\WW^U_Q, \quad
    \kk^U_t = \xx_t\WW^U_K, \quad 
    \vv^U_t = \xx_t\WW^U_V, \\
    \qq^C_t &= \xx_t\WW^C_Q, \quad 
    \kk^C_t = \xx_t\WW^C_K, \quad   
    \vv^C_t = \xx_t\WW^C_V. 
}
 
\eq{
    \QQ^U_t &= \begin{pmatrix}
    \qq^U_1 \\
    \qq^U_2 \\
    \vdots \\
    \qq^U_{t}
\end{pmatrix} = \XX^t\WW^U_Q, \quad   
    \KK^U_t = \begin{pmatrix}
    \kk^U_1 \\
    \kk^U_2 \\
    \vdots \\
    \kk^U_{t}
\end{pmatrix} = \XX^t\WW^U_K, \quad   
    \VV^U_t = \begin{pmatrix}
    \vv^U_1 \\
    \vv^U_2 \\
    \vdots \\
    \vv^U_{t}
\end{pmatrix} = \XX^t\WW^U_V, \\  
    \QQ^C_t &= \begin{pmatrix}
    \qq^C_1 \\
    \qq^C_2 \\
    \vdots \\
    \qq^C_{t}
\end{pmatrix} = \XX^t\WW^C_Q, \quad   
    \KK^C_t = \begin{pmatrix}
    \kk^C_1 \\
    \kk^C_2 \\
    \vdots \\
    \kk^C_{t}
\end{pmatrix} = \XX^t\WW^C_K, \quad   
    \VV^C_t = \begin{pmatrix}
    \vv^C_1 \\
    \vv^C_2 \\
    \vdots \\
    \vv^C_{t}
\end{pmatrix} = \XX^t\WW^C_V.  
}
And $\MM_t^C$, $\widetilde{\MM}^C_t$ and $\MM_t^U$ are $t$-by-$t$ mask matrices. 
$\MM^{C}_t$ is the $t$-by-$t$  causal mask which prevents tokens from attending to their future tokens, i.e., $[\MM^{C}_t]_{ij} = 0$ if $i \geq j$ and $[\MM^{C}_t]_{ij} = -\infty$ otherwise;   $[\widetilde{\MM}^{C}_t]_{ij} = 1$ if $i \geq j$ and $[\widetilde{\MM}^{C}_t]_{ij} = 0$ otherwise.
For CASTLE, $\MM^U_t$ is defined in~\eqref{eq:MU1} and for CASTLE-SWL, $\MM^U_t$ is defined in~\eqref{eq:MUupperwindow}.

For the projection matrices of the entire sequence $\XX^L$, we drop $L$ as in Theorem~\ref{thm:attention1} for simplicity, i,e,
\eq{
    \QQ^U &= \QQ^U_L = \XX^L \WW^U_Q,  \quad
    \KK^U = \KK^U_L = \XX^L \WW^U_K,  \quad
    \VV^U = \VV^U_L = \XX^L \WW^U_V, \\
    \QQ^C &= \QQ^C_L = \XX^L \WW^C_Q,  \quad
    \KK^C = \KK^C_L = \XX^L \WW^C_K,  \quad
    \VV^C = \VV^C_L = \XX^L \WW^C_V.                                                       
}
And $\MM^U = \MM^U_L$, $\MM^C = \MM^C_L$. 
Then, $\MM^U_t$ is a $t$-by-$t$ submatrix of $\MM^U = \MM^U_L$. Similarly, $\MM^C_t$ is also a submatrix of $\MM^C$. 

Consider when we are generating the $(t+1)$-th token. As in~\eqref{eq:Ut}
\eq{
    \UU^t = 
    \begin{pmatrix}
      \uu^t_1 \\
      \uu^t_2 \\
      \vdots \\
      \uu^t_{t}
    \end{pmatrix}
    = \sigmoid\prb{\frac{\QQ^U_t {\KK^U_t}\tp}{\sqrt{\dk}} + \MM^U_t} \VV^U_t \in \Real^{t\x \dv}
}

Then, the lookahead-key attention scores as in~\eqref{eq:stE} are 
\eql{\label{eq:stequ1}}{
    \ss^U_t &= \frac{\qq^C_{t} {\UU^t}\tp}{\sqrt{\dk}} 
    = \frac{\qq^C_{t} {\VV^U_t}\tp \pr{\sigmoid\prb{\frac{\QQ^U_t {\KK^U_t}\tp}{\sqrt{\dk}} + \MM^U_t}}\tp}{\sqrt{\dk}}.
}

We will need the following lemma to proceed.  
\begin{lemma}\label{lem:sigmoidQKM}
    For any vector $\aa\in \Real^{1\x t}$, let $\widetilde{\aa} = (\aa, \zero^{1\x (L-t)})\in \Real^{1\x L}$, where $\zero^{1\x (L - t)}$ is the all-zeros vector of size $(1, L - t)$. 
    Then, 
    \eq{
        \pr{\aa \pr{\sigmoid\prbb{\frac{\QQ^U_t {\KK^U_t}\tp}{\sqrt{\dk}} + \MM^U_t}}\tp, \zero^{1\x (L - t)}} 
        = \widetilde{\aa} \pr{\sigmoid\prbb{\frac{\QQ^U {\KK^U}\tp}{\sqrt{\dk}} + \MM^U}}\tp.   
    }
\end{lemma}
\begin{proof}[Proof of Lemma~\ref{lem:sigmoidQKM}.]
    The proof is straightforward by the fact that
    \begin{enumerate}
        \item The upper triangular entries of the transposed matrix $\pr{\sigmoid\prbb{\frac{\QQ^U {\KK^U}\tp}{\sqrt{\dk}} + \MM^U}}\tp$ are all 0 by the definition of $\MM^U$.  
        \item The matrix $\sigmoid\prbb{\frac{\QQ^U_t {\KK^U_t}\tp}{\sqrt{\dk}} + \MM^U_t}$ equals an upper-left block of the matrix $\sigmoid\prbb{\frac{\QQ^U {\KK^U}\tp}{\sqrt{\dk}} + \MM^U}$, i.e., 
        \eq{
            \sigmoid\prbb{\frac{\QQ^U_t {\KK^U_t}\tp}{\sqrt{\dk}} + \MM^U_t} 
            = \br{\sigmoid\prbb{\frac{\QQ^U {\KK^U}\tp}{\sqrt{\dk}} + \MM^U}}_{1:t, 1:t},  
        }
        where for any matrix $\AA$, $\AA_{1:t, 1:t}$ refers to its top-left $t$-by-$t$ submatrix.  
        
    \end{enumerate}
\end{proof}

Define the vector $\aa_{t}\in \Real^{1\x L}$ as
$
    [\aa_{t}]_i = \prb{\qq^C_{t} \VV^U_t}_i
$
if $1\leq i\leq t$ and $[\aa_{t}]_i = 0$ otherwise.

Denote $\widetilde{\ss}^U_t = \pr{\ss^U_t, \zero^{1\x (L - t)}}$.  
Then, by combining~\eqref{eq:stequ1} with Lemma~\ref{lem:sigmoidQKM}, we have
\eq{
    \widetilde{\ss}^U_t &= \pr{\ss^U_t, \zero^{1\x (L - t)}} = \pr{\frac{\qq^C_{t} {\VV^U_t}\tp \pr{\sigmoid\prbb{\frac{\QQ^U_t {\KK^U_t}\tp}{\sqrt{\dk}} + \MM^U_t}}\tp}{\sqrt{\dk}}, \zero^{1\x (L - t)}} \\  
    &= \frac{\aa_{t} \pr{\sigmoid\prbb{\frac{\QQ^U {\KK^U}\tp}{\sqrt{\dk}} + \MM^U}}\tp}{\sqrt{\dk}}.  
}

Since $\VV^U_t$ equals the submatrix which consists of first $t$ rows of $\VV^U$, $[\aa_{t}]_i = \prB{\qq^C_{t} {\VV^U}\tp}_i$ for $1\leq i\leq t$. 
Then, by stacking $\prb{\aa_j}$ together, we have
\eq{
    \begin{pmatrix}
        \aa_1 \\
        \aa_2 \\
        \vdots \\
        \aa_{L}
    \end{pmatrix}
    = \pr{\QQ^C {\VV^U}\tp} \odot \widetilde{\MM}^C,     
} 
where $\widetilde{\MM}^C$ is defined in Theorem~\ref{thm:attention1} with $\widetilde{\MM}^{C}_{ij} = 1$ if $i \geq j$ and $\widetilde{\MM}^{C}_{ij} = 0$ otherwise.  

We concatenate $\pr{\widetilde{\ss}^U_t}_{1\leq t\leq L}$ in an $L$-by-$L$ matrix.   
Then, 
\eql{\label{eq:SE1}}{
    \begin{pmatrix}
        \widetilde{\ss}^U_1 \\
        \widetilde{\ss}^U_2 \\
        \vdots \\
        \widetilde{\ss}^U_L
    \end{pmatrix}
    &= 
    \frac{1}{\sqrt{\dk}}
    \begin{pmatrix}
        \aa_1 \\
        \aa_2 \\
        \vdots \\
        \aa_{L}
    \end{pmatrix}
    \pr{\sigmoid\prbb{\frac{\QQ^U {\KK^U}\tp}{\sqrt{\dk}} + \MM^U}}\tp \\  
    &= \pr{\frac{\QQ^C {\VV^U}\tp}{\sqrt{\dk}} \odot \widetilde{\MM}^C} \pr{\sigmoid\prbb{\frac{\QQ^U {\KK^U}\tp}{\sqrt{\dk}} + \MM^U}}\tp \\
    &= \SSU,   
} 
where $\SSU$ is given in~\eqref{eq:R1}.  

The causal-key attention scores $\ss^C_t$ in~\eqref{eq:stC} is
\eq{
    \ss^C_t = \frac{\qq^C_{t}{\KK^C_t}\tp}{\sqrt{\dk}}.  
}
We also denote $\widetilde{\ss}^C_t = \pr{\ss^C_t, (-\infty)^{1\x (L - t)}}$, where $(-\infty)^{1\x (L - t)}$ is a $(L - t)$-dimensional vector with all entries equaling $-\infty$.   
Then, by concatenating $\pr{\widetilde{\ss}^C_t}_{1\leq t\leq L}$ into $\SS^C\in \Real^{L\x L}$, we have
\eql{\label{eq:SC1}}{
    \SS^C = 
    \begin{pmatrix}
        \widetilde{\ss}^C_1 \\
        \widetilde{\ss}^C_2 \\
        \vdots \\
        \widetilde{\ss}^C_L
    \end{pmatrix}
    = \frac{\QQ^C {\KK^C}\tp}{\sqrt{\dk}} + \MM^C.    
} 
Then, the outputs $\Atten(\XX^L)$ satisfies
\eq{
    \Atten\pr{\XX^L} &= 
    \begin{pmatrix}
        \atten\pr{\XX^1} \\
        \atten\pr{\XX^2} \\
        \vdots \\
        \atten\pr{\XX^L}
    \end{pmatrix} 
    =
    \begin{pmatrix}
        \softmax\pr{\ss^C_1 - \silu(\ss^U_1)}\VV^C_1 \\
        \softmax\pr{\ss^C_2 - \silu(\ss^U_2)}\VV^C_2 \\ 
        \vdots \\  
        \softmax\pr{\ss^C_L - \silu(\ss^U_L)}\VV^C_L 
    \end{pmatrix} \\  
    &= 
    \begin{pmatrix}
        \softmax\pr{\widetilde{\ss}^C_1 - \silu(\widetilde{\ss}^U_1)} \\
        \softmax\pr{\widetilde{\ss}^C_2 - \silu(\widetilde{\ss}^U_2)} \\ 
        \vdots \\  
        \softmax\pr{\widetilde{\ss}^C_L - \silu(\widetilde{\ss}^U_L)} 
    \end{pmatrix} 
    \VV^C \\  
    &=
    \rsoftmax\pr{\SS^C - \silu\pr{\SS^U}} \VV^C \\  
    &= \rsoftmax\pr{\frac{\QQ^C{\KK^C}\tp}{\sqrt{\dk}} + \MM^C - \silu\pr{\frac{\SSU}{\sqrt{\dk}}}} \VV^C,   
} 
where the last inequality above is from~\eqref{eq:SE1} and~\eqref{eq:SC1}. 

This completes the proof of Theorem~\ref{thm:attention1}. 

\section{Further Details on Multi-Head CASTLE}\label{sec:prfmultiheadattention} 
\noindent\textbf{Forward pass for multi-head CASTLE.}
Denote $n = \nhead$.  
Given contextualized representations $\XX^L$, 
for each head, we can get $\Atten_i(\XX^L)$ as in~\eqref{eq:Attentiontrain}. 
Then, the outputs of multi-head CASTLE can be obtained by
\eq{
    \Multiheadatten(\XX^L) = \concat\pr{\Atten_1(\XX^L), \dots, \Atten_n(\XX^L)} \WW^O \in \Real^{L\x \dv}.  
}

\noindent\textbf{Parameter count of multi-head CASTLE.} 
In each head, the learnable parameters are $\WW^U_Q$, $\WW^U_K$, $\WW^U_V$, $\WW^C_Q$, $\WW^C_K$, $\WW^C_V\in \Real^{\dhidden\x \dk}$. 
These parameters sum up to $6 \nhead \dhidden \dk$. 

The matrix $\WW^O$ has $\nhead \dk \dhidden $ parameters.
Thus, the multi-head CASTLE has $7\nhead\dk\dhidden$ learnable parameters. 

Multi-head CASTLE-SWL has identical formula and parameter counts with multi-head CASTLE and is omitted for clarity.   

\noindent\textbf{Parameter count of standard causal attention.}
The standard causal attention as in~\eqref{eq:standardcausal2} has learnable parameters $\WW_Q$, $\WW_K$, $\WW_V\in \Real^{\dhidden\x \dk}$ for each head and $\WW^O \in \Real^{n\dk \x \dhidden}$. 
These parameters sum up to $4\nhead\dk\dhidden$.

\section{Efficient Parallel Training Algorithm and Proof of Theorem~\ref{thm:algU1}}\label{sec:prfthmalgU1}

\subsection{Forward Pass}

\begin{algorithm}
    \caption{Efficient parallel forward pass}\label{alg:train}
    \KwRequire{$\QQ^U$, $\KK^U$, $\VV^U$, $\QQ^C$, $\KK^C$, $\VV^C\in \Real^{L\x \dk}$ }
    \# Initialization \;
    Initialize $\DD = \zero^{\dk \x L}$, $\OO = \zero^{L\x \dk}$, 
    $\ell = \zero^{L\x 1}$ and $\mm = (-\infty)^{L\x 1}$ in HBM \;

    \# Diagonal blocks \;
    \For{$j=0, \dots, N-1$ (in parallel)}{
        Load $\QQ^C_{T_j, :}$, $\KK^C_{T_j, :}$, $\QQ^U_{T_j, :}$, $\KK^U_{T_j, :}$, $\VV^U_{T_j, :}$ from HBM to on-chip SRAM \;
        On chip, compute $\SS^U_{T_j, T_j}$ as in~\eqref{eq:SSEdiagonal} \;
        On chip, compute $\AA_{T_j, T_j}$ as in~\eqref{eq:ATij} \;  
        Implement online softmax update for block $(T_j, T_j)$ by calling Algorithm~\ref{alg:softmaxupdate}
    }

    \# $1$-st, $\cdots$, $(N-1)$-th off-diagonal blocks \;  
    \For{$k=1, \dots, N-1$ (sequential)}{
        \For{$j=0, \dots, N - k - 1$ (in parallel)}{
            Load $\DD_{:, T_j}$, $\QQ^U_{T_{j}, :}$, $\QQ^C_{T_{j+k}}$, $\KK^U_{T_{j+k-1}, :}$, $\KK^U_{T_{j+k}, :}$, $\VV^U_{T_{j+k-1}, :}$, $\VV^U_{T_{j+k}, :}$ from HBM to on-chip SRAM \;
            On chip, update $\DD_{:, T_j}$ as in~\eqref{eq:Dupdate} \;
            On chip, compute $\SS^U_{T_{j+k}, T_j}$ as in~\eqref{eq:Rupdate} \;
            Write $\DD_{:, T_j}$ to HBM \;
            On chip, compute $\AA_{T_{j+k}, T_j}$ as in~\eqref{eq:ATij} \;
            Implement online softmax update for block $(T_{j+k}, T_j)$ by calling Algorithm~\ref{alg:softmaxupdate}
        }
    }
    Compute $\OO \leftarrow \diag{\ell}\inv\OO \in \Real^{L\x d}$, $\mm \leftarrow \mm + \log(\ell) \in \Real^{L \x 1}$\; 
    Save $\mm\in \Real^{L\x 1}$, $\DD\in \Real^{\dk \x L}$ for backward pass \;  
    Return the output $\OO\in \Real^{L\x \dk}$  
    
\end{algorithm}

\begin{algorithm}
    \caption{Online softmax update for block $(T_i, T_j)$}\label{alg:softmaxupdate}
    \KwRequire{$\AA_{T_i, T_j}$ on chip, $\ell$, $\mm$, $\OO$, $\VV^C$ in HBM}
    Load $\ell_{T_i}$ and $\mm_{T_i}$ from HBM to on-chip SRAM \;
    On chip, compute $\mm_{T_i}^{\rm new} = \max(\mm_{T_i}, \text{row\_max}(\AA_{T_i, T_j}))$ \;
    On chip, compute $\widetilde{\PP}_{T_i, T_j} = \exp(\AA_{T_i, T_j} - \mm_{T_i}^{\rm new})$ \;
    On chip, compute $\ell_{T_i}^{\rm new} = e^{\mm_{T_i} - \mm_{T_i}^{\rm new}}\ell_{T_i} + \text{row\_sum}(\widetilde{\PP}_{T_i, T_j})$ \;
    Write $\OO_{T_i, :} \leftarrow \diag{e^{\mm_{T_i} - \mm_{T_i}^{\rm new}}} \OO_{T_i, :} + \widetilde{\PP}_{T_i, T_j} \VV^C_{T_j, :}$ to HBM \;
    Write $\ell_{T_i} \leftarrow \ell_{T_i}^{\rm new}$, $\mm_{T_i} \leftarrow \mm_{T_i}^{\rm new}$ to HBM
\end{algorithm}

\begin{figure}[ht!]
    \centering
    \includegraphics[width=0.68\linewidth]{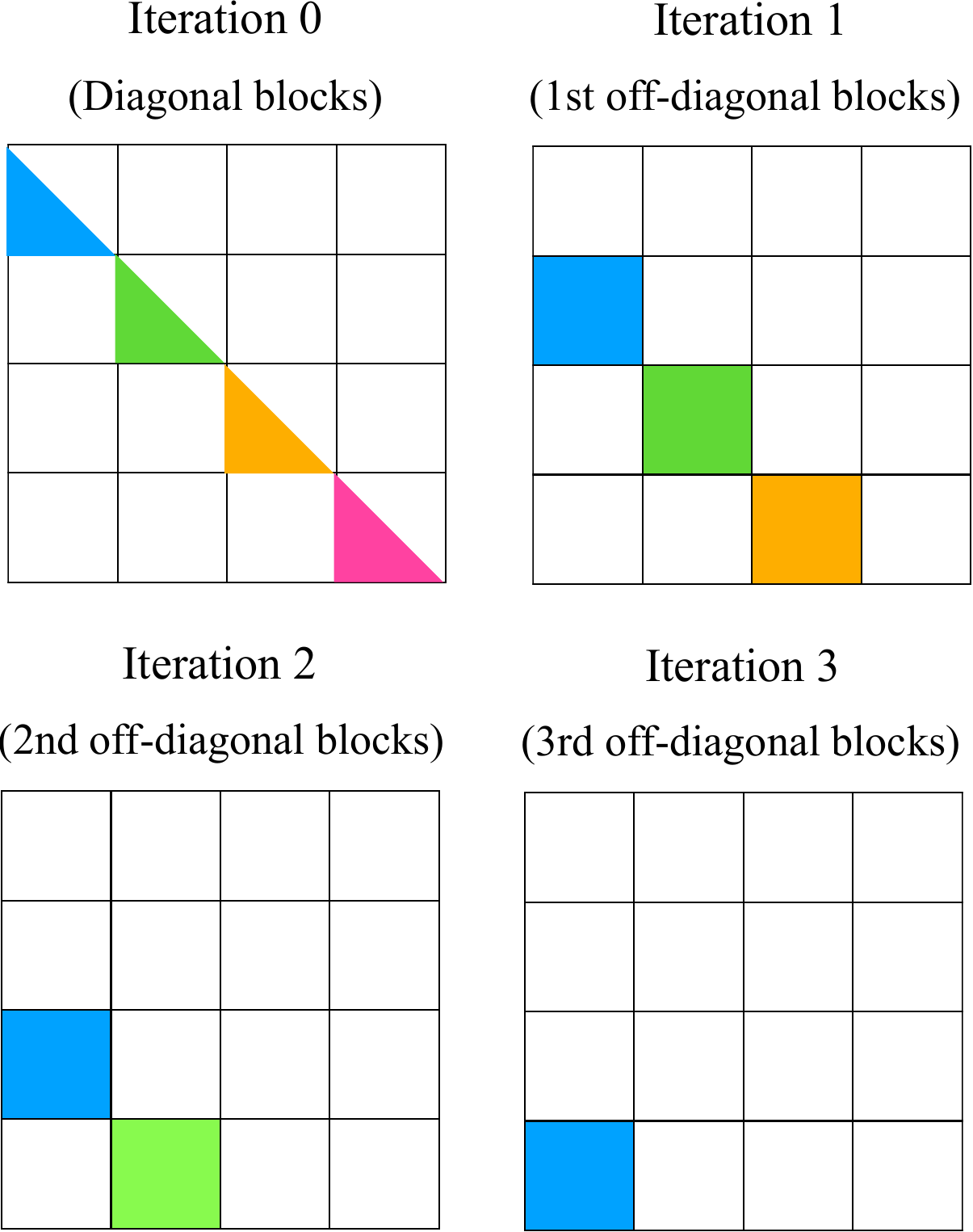}
    \caption{Parallel scheme of Algorithm~\ref{alg:train} (forward pass). We begin by computing the diagonal blocks of the attention score matrix $\AA$ (Iteration 0). In each subsequent iteration $k$, the $k$-th off-diagonal blocks are computed. Blocks with different colors represent different kernel instances. In each iteration, each kernel instance is responsible for computing a single block of $\AA$ and applying online softmax (Algorithm~\ref{alg:softmaxupdate}) to it. Kernel instances within the same iteration are launched in parallel, while the iterations are executed sequentially. The parallel scheme of Algorithm~\ref{alg:traingrad} (backward pass) is the reverse of the forward pass's parallel scheme. }
    \label{fig:workpartition}
\end{figure}

Theorem~\ref{thm:attention1} has shown a matrix form of $\Atten\prn{\XX^L}$ as in~\eqref{eq:Attentiontrain}.
However, computing $\Atten(\XX^L)$ directly from~\eqref{eq:Attentiontrain} still need $O(L^3)$ computational costs because to compute $\SSU$, we need matrix multiplication between $L$-by-$L$ matrices in~\eqref{eq:R1}. 
In this section, we give an efficient algorithm that enables efficient parallel training and reduces computational complexity to $O(L^2 \dk)$. 

We first introduce the notations in this section. 
We divide the sequence $\dr{1, \dots, L}$ into blocks of size $B$. 
For simplicity, we assume that $L$ is divisible by $B$. 
Let $N = \frac{L}{B}$ and $T_i = \dr{i * B + 1, \dots, (i+1)*B}$ for $0\leq i\leq N - 1$. 
Then, $\dr{1, 2, \dots, L} = \cup_{i=0}^{N-1} T_i$. 

For any matrix $\MM$, we use $\MM_{T_i, T_j}$ to denote the submatrix whose row and column indexes are in $T_i$ and $T_j$, respectively.
$\MM_{T_i, :}$ refers to the submatrix whose row indexes are in $T_i$. 
Analogously, $\MM_{:, T_j}$ refers to the submatrix whose column indexes are in $T_j$. 

Denote the attention score matrix $\AA$ by
\eql{\label{eq:defatten1}}{
    \AA = \frac{\QQ^C{\KK^C}\tp}{\sqrt{\dk}} + \MM^C - \silu\pr{\SS^U}.
}
Then, 
\eq{
    \Atten(\XX^L) = \rsoftmax(\AA)\VV^C  
}
and 
for any $0\leq i, j\leq N-1$, the block $\AA_{T_i, T_j}$ satisfies
\eql{\label{eq:ATij}}{
    \AA_{T_i, T_j} = \frac{\QQ^C_{T_i, :}{\KK^C_{T_j, :}}\tp}{\sqrt{\dk}} + \MM^C_{T_i, T_j} - \silu\pr{\SS^U_{T_i, T_j}}.  
}
As in FlashAttention-2~\citep{dao2023flashattention}, we compute each block of $\AA$ and apply online softmax (Algorithm~\ref{alg:softmaxupdate}) on it to obtain $\Atten(\XX^L)$. 
The first term $\frac{\QQ^C_{T_i, :}{\KK^C_{T_j, :}}\tp}{\sqrt{\dk}} $ in         $\AA_{T_i, T_j}$ as given by~\eqref{eq:ATij} can be computed similarly to FlashAttention-2. 
We mainly focus on the second term $-\silu\pr{\SSU_{T_i, T_j}}$ which can incur a total computational complexity of $O(L^3 + L^2d)$ if we apply~\eqref{eq:R1} directly. 
Notice that as in~\eqref{eq:R1}, 
\eq{
    \SSU = \pr{\frac{\QQ^C {\VV^U}\tp}{\sqrt{\dk}} \odot \widetilde{\MM}^{C}} \pr{\sigmoid\pr{\frac{\QQ^U {\KK^U}\tp}{\sqrt{\dk}} + \MM^U}}\tp,  
}
where the term $\pr{\frac{\QQ^C {\VV^U}\tp}{\sqrt{\dk}} \odot \widetilde{\MM}^{C}}$ is a low-rank matrix multiplied by a mask matrix because $\QQ^C {\VV^U}\tp$ has rank at most $\dk$ which is normally much smaller than sequence length $L$. 
This motivates our algorithm to compute $\SSU$ with computational complexity of $O(L^2d)$ while still enabling parallel training.  

First, as the upper triangular entries in $\SSU$ are all $0$,
we only need to focus on lower triangular entries in $\SSU$. 
We divide the computation of $\SSU$ into the following parts
\begin{itemize}
    \item \textbf{diagonal blocks:} $\SSU_{T_j, T_j}$ with $0\leq j\leq N - 1$. 
    \item \textbf{k-th off-diagonal blocks:} $\SSU_{T_{j+k, T_j}}$ with $0\leq j\leq N - 1 - k$
    
\end{itemize}

By the definition of $\SSU$, we can write $\SSU_{T_{j+k}, T_j}$ in a blockwise way as follows
\eql{\label{eq:RTkjTj1}}{
    & \SSU_{T_{j+k}, T_j} = \QQ^C_{T_{j+k}, :} \sum_{i=j}^{j+k-1} \frac{{\VV^U_{T_i, :}}\tp}{\sqrt{\dk}} \pr{\sigmoid\prbb{\frac{\QQ^U_{T_j, :}{\KK^U_{T_{i}, :}}\tp}{\sqrt{\dk}} + \MM^U_{T_j, T_i}}}\tp \\
    & + \pr{\frac{\QQ^C_{T_{j+k}, :} {\VV^U_{T_{j+k}, :}}\tp}{\sqrt{\dk}} \odot \widetilde{\MM}^C_{T_{j+k}, T_{j+k}}} \pr{\sigmoid\prbb{\frac{\QQ^U_{T_{j}, :}{\KK^U_{T_{j+k}, :}}\tp}{\sqrt{\dk}} + \MM^U_{T_{j}, T_{j+k}}}}\tp.       
}
However, computing each $\SSU_{T_{j+k}, T_j}$ this way will lead to $O(L^3)$ computational cost in total. We need more efficient way to reduce training costs.

The key observation in the development of this efficient parallel training algorithm is that the matrix $\pr{\QQ^C {\VV^U}\tp} \odot \widetilde{\MM}^{C}$ is a low-rank matrix multiplied by a mask because $\QQ^C {\VV^U}\tp$ has rank equal to or less than the head dimension $\dk$ rather than the sequence length $L$. 
This enables us to compute $\SSU$ defined in~\eqref{eq:R1} with lower computational costs. 
More specifically, we will rely on the following auxiliary variable to reduce computational costs.

We first init an auxiliary variable 
$\DD^{(0)} = \zero^{\dk\x L}$. 
Then, when computing the $k$-th off-diagonal blocks, we will maintain the following relation that 
\eql{\label{eq:Dk1}}{
    \DD^{(k)}_{:, T_j} = \sum_{i=j}^{j+k-1} {\VV^U_{T_i, :}}\tp \pr{\sigmoid\prbb{\frac{\QQ^U_{T_j, :}{\KK^U_{T_{i}, :}}\tp}{\sqrt{\dk}} + \MM^U_{T_j, T_i}}}\tp.    
}

We will compute the blocks of $\SSU$ in the following order: diagonal blocks, 1st off-diagonal blocks,  2nd off-diagonal blocks, $\cdots$, $(N-1)$-th off-diagonal blocks. 

\noindent\textbf{Diagonal blocks of $\SSU$.} 
For any $0\leq j\leq N-1$, 
\eql{\label{eq:SSEdiagonal}}{
    \SSU_{T_j, T_j} = \pr{\frac{\QQ^C_{T_j, :} {\VV^U_{T_j, :}}\tp}{\sqrt{\dk}} \odot \widetilde{\MM}^C_{T_j, T_j}} \pr{\sigmoid\prbb{\frac{\QQ^U_{T_j, :}{\KK^U_{T_j, :}}\tp}{\sqrt{\dk}} + \MM^U_{T_j, T_j}}}\tp.   
}

\noindent\textbf{1st off-diagonal blocks of $\SSU$.}
For any $0\leq j\leq N-2$, 
update $\DD^{(1)}_{:, T_j}$ as
\eq{
    \DD^{(1)}_{:, T_j} = {\VV^U_{T_j, :}}\tp \pr{\sigmoid\prbb{\frac{\QQ^U_{T_j, :}{\KK^U_{T_j, :}}\tp}{\sqrt{\dk}} +  \MM^U_{T_j, T_j}}}\tp
}
This satisfies~\eqref{eq:Dk1} with $k=1$.  
Then, it follows from~\eqref{eq:RTkjTj1} that  
\eq{
    & \SSU_{T_{j+1}, T_j} = \frac{\QQ^C_{T_{j+1}, :} \DD^{(1)}_{:, T_j}}{\sqrt{\dk}} \\
    & + \pr{\frac{\QQ^C_{T_{j+1}, :} {\VV^U_{T_{j+1}, :}}\tp}{\sqrt{\dk}} \odot \widetilde{\MM}^C_{T_{j+1}, T_{j+1}}} \pr{\sigmoid\prbb{\frac{\QQ^U_{T_{j}, :}{\KK^U_{T_{j+1}, :}}\tp}{\sqrt{\dk}} + \MM^U_{T_{j}, T_{j+1}}}}\tp.       
} 

\noindent\textbf{The $k$-th off-diagonal blocks of $\SSU$.}
If we have already computed the $(k-1)$-th off-diagonal blocks and $\DD^{(k-1)}$ satisfies~\eqref{eq:Dk1} with $k-1$, the $k$-th diagonal blocks can be computed in the following way. 
First, we update $\DD^{(k)}$ as follows
\eql{\label{eq:Dupdate}}{
    \DD^{(k)}_{:, T_j} = \DD^{(k-1)}_{:, T_j} + {\VV^U_{T_{j+k-1}, :}}\tp \pr{\sigmoid\prbb{\frac{\QQ^U_{T_j, :}{\KK^U_{T_{j+k-1}, :}}\tp}{\sqrt{\dk}} + \MM^U_{T_j, T_{j+k-1}}}}\tp.  
} 
By induction hypothesis~\eqref{eq:Dk1}, $\DD^{(k)}_{:, T_j}$ also satisfies~\eqref{eq:Dk1} with $k$. 
Then, it follows by~\eqref{eq:RTkjTj1} that 
\eql{\label{eq:Rupdate}}{
    & \SSU_{T_{j+k}, T_j} = \frac{\QQ^C_{T_{j+k}, :}\DD^{(k)}_{:, T_j}}{\sqrt{\dk}}  \\
    & + \pr{\frac{\QQ^C_{T_{j+k}, :} {\VV^U_{T_{j+k}, :}}\tp}{\sqrt{\dk}} \odot \widetilde{\MM}^C_{T_{j+k}, T_{j+k}}} \pr{\sigmoid\prbb{\frac{\QQ^U_{T_{j}, :}{\KK^U_{T_{j+k}, :}}\tp}{\sqrt{\dk}} +  \MM^U_{T_{j}, T_{j+k}}}}\tp.       
} 

We can compute all $k$-th off-diagonal blocks in the above way. 
When computing each $\SSU_{T_{j+k}, T_j}$, we only need to update $\DD^{(k)}$ and then compute $\SSU_{T_{j+k}, T_j}$ as~\eqref{eq:Rupdate}. 
Both~\eqref{eq:Dupdate} and~\eqref{eq:Rupdate} take $O(L^2d)$ FLOPs. 

Next, we describe the design of a parallel algorithm. The parallelization scheme must satisfy the following requirements:

\begin{itemize}
\item $\SSU_{T_{j+k}, T_j}$ must be computed after $\SSU_{T_{j+k-1}, T_j}$ because computing $\SSU_{T_{j+k}, T_j}$ in~\eqref{eq:Rupdate} requires $\DD^{(k)}$ which is derived by updating $\DD^{(k-1)}$ in~\eqref{eq:Dupdate}. Therefore, $\AA_{T_{j+k}, T_j}$ should be computed after $\AA_{T_{j+k-1}, T_j}$.  
\item To ensure correctness, online softmax cannot be applied simultaneously to $\AA_{T_i, T_j}$ and $\AA_{T_i, T_k}$ for $j \ne k$.
\end{itemize}

To meet these constraints, we launch kernel instances to compute attention outputs with respect to blocks in the following order: starting with the diagonal blocks, followed by the first off-diagonal blocks, then the second, and so on up to the $(N-1)$-th off-diagonal blocks. This parallel execution strategy for Algorithm~\ref{alg:train} is illustrated in Figure~\ref{fig:workpartition}.

\subsection{Backward Pass}
In this section, we introduce the backward pass algorithm for efficient parallel training. 
It is mainly derived by reversing the forward pass. 

Recall that in the forward pass, we compute the blocks in the following order: diagonal blocks, 1st off-diagonal blocks, 2nd off-diagonal blocks, $\cdots$, $(N-1)$-th off-diagonal blocks. 
Then, in the backward pass, we compute the derives in the inverse order as follows: $(N-1)$-th off-diagonal blocks, $(N-2)$-th off-diagonal blocks, $\cdots$, $1$-st off-diagonal blocks, diagonal blocks. 

\begin{algorithm}
    \caption{Efficient parallel backward pass}\label{alg:traingrad}
    \KwRequire{$\de\OO\in \Real^{L\x \dhidden}$}
    \# Initialization \;
    Initialize $\de\DD = \zero^{\dk \x L}$, $\de\QQ^U =\zero^{L\x \dk}$, $\de\KK^U =\zero^{L\x \dk}$, $\de\VV^U =\zero^{L\x \dk}$, $\de\QQ^C =\zero^{L\x \dk}$, $\de\KK^C =\zero^{L\x \dk}$, $\de\VV^C =\zero^{L\x \dk}$\;
    Let $\mm\in \Real^{L\x 1}$, $\DD\in \Real^{\dk \x L}$ be the tensors saved in forward pass (Algorithm~\ref{alg:train}) \;  

    \# Preprocess \;
    Compute $\Delta\in \Real^{L\x 1}$ with $\Delta_i = \de\OO_{i, :}\tp \OO_{i, :}$ \;  

    \# $(N-1)$-th, $\cdots$, $1$st off-diagonal blocks \;  
    \For{$k=N-1, \dots, 1$ (sequential)}{
        \For{$j=0, \dots, N - k - 1$ (in parallel)}{
            Compute $\SS^U_{T_{j+k}, T_j}$ as in~\eqref{eq:Rupdate} and $\AA_{T_{j+k}, T_j}$ as in~\eqref{eq:ATij} \; 
            Compute $\PP_{T_{j+k}, T_j}$ and $\de\PP_{T_{j+k}, T_j}$ as in~\eqref{eq:Pbwd} and~\eqref{eq:dPbwd} \; 
            Compute $\de\SS^C_{T_{j+k}, T_j}$ and $\de\SSU_{T_{j+k}, T_j}$ as in~\eqref{eq:dSSCbwd} and~\eqref{eq:dSSUbwd} \;  
            Update $\de\VV^C_{T_j, :}$, $\de\QQ^C_{T_{j+k}, :}$, $\de\KK^C_{T_j, :}$ as in~\eqref{eq:dVbwd} and~\eqref{eq:dQKCbwd} \; 
            Update $\de\QQ^C_{T_{j+k}, :}$, $\de\DD^{(k)}_{:, T_j}$, $\de\VV^U_{T_{j+k}, :}$, $\de\QQ^U_{T_{j}, :}$, $\KK^U_{T_{j+k}, :}$ as in~\eqref{eq:dQKDVSSUbwd} \; 
            Update $\de\VV^U_{T_{j+k-1}, :}$, $\de\QQ^U_{T_j, :}$, $\de\KK^U_{T_{j+k-1}, :}$ as in~\eqref{eq:dVQDbwd} \; 
            Compute $\DD^{(k-1)}_{:, T_j}$ as in~\eqref{eq:Dupdate2}  
        }
    }

    \# Diagonal blocks \;
    \For{$j=0, \dots, N-1$ (in parallel)}{
        Compute $\SS^U_{T_{j}, T_j}$ as in~\eqref{eq:SSEdiagonal} and $\AA_{T_{j}, T_j}$ as in~\eqref{eq:ATij} \; 
        Compute $\PP_{T_{j}, T_j}$ and $\de\PP_{T_{j}, T_j}$ as in~\eqref{eq:Pbwddiagonal} and~\eqref{eq:dPbwddiagonal} \; 
        Compute $\de\SS^C_{T_{j}, T_j}$ and $\de\SSU_{T_{j}, T_j}$ as in~\eqref{eq:dSSCbwddiagonal} and~\eqref{eq:dSSUbwddiagonal} \; 
        Update $\de\VV^C_{T_j, :}$, $\de\QQ^C_{T_j, :}$, $\de\VV^U_{T_j, :}$, $\de\QQ^U_{T_j, :}$, $\de\KK^U_{T_j, :}$ as in~\eqref{eq:dVbwddiagonal} and~\eqref{eq:dQKVdiagonal}  
    }
    Return the gradients $\de\QQ^U$, $\de\KK^U$, $\de\VV^U$, $\de\QQ^C$, $\de\KK^C$, $\de\VV^C \in \Real^{L\x \dk}$  
    
\end{algorithm}

As in~\citep{dao2022flashattention}, we first preprocess $\Delta \in \Real^{L \x 1}$ with
\eq{
    \Delta_{i} = \de\OO_{i, :}\tp \OO_{i, :}.  
}
Let $\DD^{(N-1)}$ be the $\DD$ saved for backward pass in Algorithm~\ref{alg:train}. 
Also, let $\mm$ be the vector $\mm$ saved for backward pass in Algorithm~\ref{alg:train}. 
Set $\de\DD^{(N-1)} = \zero^{\dk \x L}$.  

Then, we iterate over $k = N-1, \dots, 1$. 
In each iteration, we compute the corresponding gradients and update $\DD^{(k)}$ to $\DD^{(k-1)}$ inversely as in the forward pass. 
Thus, for each $k$, before we launch kernels for the $k$-th off-diagonal blocks, we already have $\DD^{(k)}$.  

\noindent\textbf{$k$-th off-diagonal blocks.} 
For any $0\leq j \leq N-1-k$, 
we first compute $\SS^U_{T_{j+k}, T_j}$ as in~\eqref{eq:Rupdate} and $\AA_{T_{j+k}, T_j}$ as in~\eqref{eq:ATij}. 
Recall that in the end of the forward pass (Algorithm~\ref{alg:train}), we have set $\mm \leftarrow \mm + \log(\ell)$.  
Then, the attention weights can be computed by  
\eql{\label{eq:Pbwd}}{
    \PP_{T_{j+k}, T_j} = \exp\pr{\AA_{T_{j+k}, T_k} - \mm_{T_{j+k}}}. 
}
Compute the gradient of attention weights
\eql{\label{eq:dPbwd}}{
    \de\PP_{T_{j+k}, T_j} = \de\OO_{T_{j+k}, :} \VV_{T_j, :}\tp.  
}
Then, update the gradient of $\VV^C_{T_j, :}$ as follows
\eql{\label{eq:dVbwd}}{
    \de\VV^C_{T_j, :} \leftarrow \de\VV^C_{T_j, :} + \PP_{T_{j+k}, T_j}\tp \de\OO_{T_{j+k}, :}.    
} 
Then, compute the gradients of attention weights 
\eql{\label{eq:dSSCbwd}}{
    \de\SS^C_{T_{j+k}, T_j} = \PP_{T_{j+k}, T_j} \odot \pr{\de\PP_{T_{j+k}, T_j} - \Delta_{T_{j+k}}} 
}
and 
\eql{\label{eq:dSSUbwd}}{
    \de\SSU_{T_{j+k}, T_j} = - \nabla\silu\prb{\SSU_{T_{j+k}, T_j}} \odot \PP_{T_{j+k}, T_j} \odot \pr{\de\PP_{T_{j+k}, T_j} - \Delta_{T_{j+k}}}.  
}
Then, we update the gradients of $\QQ^C_{T_{j+k}, :}$, $\KK^C_{T_j, :}$ through the backward pass of~\eqref{eq:ATij} as follows  
\eql{\label{eq:dQKCbwd}}{
    \de\QQ^C_{T_{j+k}, :} \leftarrow& \de\QQ^C_{T_{j+k}, :} + \frac{\de\SS^C_{T_{j+k}, T_j} \KK^C_{T_j, :}}{\sqrt{\dk}}, \\
    \de\KK^C_{T_j, :} \leftarrow& \de\KK^C_{T_j, :} + \frac{{\de\SS^C_{T_{j+k}, T_j}}\tp \QQ^C_{T_{j+k}, :}}{\sqrt{\dk}}.  
}
Then, we update the gradients of $\QQ^C_{T_{j+k}, :}$, $\DD^{(k)}_{:, T_j}$, $\VV^U_{T_{j+k}, :}$, $\QQ^U_{T_{j}, :}$, $\KK^U_{T_{j+k}, :}$ through the backward pass of~\eqref{eq:Rupdate} as follows
\eql{\label{eq:dQKDVSSUbwd}}{
    &\de\QQ^C_{T_{j+k}, :} \leftarrow \de\QQ^C_{T_{j+k}, :} + \frac{\de\SSU_{T_{j+k}, T_j} {\DD^{(k)}_{:, T_j}}\tp}{\sqrt{\dk}} +  \frac{\EE\VV^U_{T_{j+k}, :}}{\sqrt{\dk}} \\
    &\de\DD^{(k)}_{:, T_j} \leftarrow \de\DD^{(k)}_{:, T_j} + \frac{{\QQ^C_{T_{j+k}, :}}\tp \de\SSU_{T_{j+k}, T_j}}{\sqrt{\dk}} \\
    &\de\VV^U_{T_{j+k}, :} \leftarrow \de\VV^U_{T_{j+k}, :} +  \frac{\EE\tp \QQ^C_{T_{j+k}, :}}{\sqrt{\dk}} \\ 
    &\de\QQ^U_{T_{j}, :} \leftarrow \de\QQ^U_{T_{j}, :} +  \frac{\FF \KK^U_{T_{j+k}, :}}{\sqrt{\dk}} \\ 
    &\de\KK^U_{T_{j+k}, :} \leftarrow \de\KK^U_{T_{j+k}, :} + \frac{\FF\tp \QQ^U_{T_{j}, :}}{\sqrt{\dk}},    
} 
where auxiliary matrices 
\eq{
    \EE &= \pr{\de\SSU_{T_{j+k}, T_j} \sigmoid\prbb{\frac{\QQ^U_{T_{j}, :}{\KK^U_{T_{j+k}, :}}\tp}{\sqrt{\dk}} + \MM^U_{T_{j}, T_{j+k}}}} \odot \widetilde{\MM}^C_{T_{j+k}, T_{j+k}}, \\
    \FF &= \pr{\pr{\de\SSU_{T_{j+k}, T_j}}\tp \pr{\frac{\QQ^C_{T_{j+k}, :} {\VV^U_{T_{j+k}, :}}\tp}{\sqrt{\dk}} \odot \widetilde{\MM}^C_{T_{j+k}, T_{j+k}}}} \\
    &\quad\quad \odot \nabla \sigmoid\prbb{\frac{\QQ^U_{T_{j}, :}{\KK^U_{T_{j+k}, :}}\tp}{\sqrt{\dk}} + \MM^U_{T_{j}, T_{j+k}}}.  
}
Then, we update the gradients of $\VV^U_{T_{j+k-1}, :}$, $\QQ^U_{T_j, :}$, $\KK^U_{T_{j+k-1}, :}$ from the backward pass of~\eqref{eq:Dupdate} as follows
\eql{\label{eq:dVQDbwd}}{
    &\de\VV^U_{T_{j+k-1}, :} \leftarrow \de\VV^U_{T_{j+k-1}, :} \\
    & \quad\quad\quad\quad\quad\quad + \prbb{\sigmoid\prbb{\frac{\QQ^U_{T_j, :}{\KK^U_{T_{j+k-1}, :}}\tp}{\sqrt{\dk}} + \MM^U_{T_j, T_{j+k-1}}}}\tp \pr{\de\DD^{(k)}_{:, T_j}}\tp \\
    &\de\QQ^U_{T_j, :} \leftarrow \de\QQ^U_{T_j, :} +  \frac{\GG \KK^U_{T_{j+k-1}, :}}{\sqrt{\dk}} \\
    &\de\KK^U_{T_{j+k-1}, :} \leftarrow \de\KK^U_{T_{j+k-1}, :} + \frac{\GG\tp \QQ^U_{T_j, :}}{\sqrt{\dk}}
}
where the auxiliary matrix
\eq{
    \GG = \pr{\pr{\de\DD^{(k)}_{:, T_j}}\tp {\VV^U_{T_{j+k-1}, :}}\tp}   \odot \nabla \sigmoid\prbb{\frac{\QQ^U_{T_j, :}{\KK^U_{T_{j+k-1}, :}}\tp}{\sqrt{\dk}} + \MM^U_{T_j, T_{j+k-1}}}.    
}
After updating gradients, we set $\de\DD^{(k-1)}_{:, T_j} \leftarrow \de\DD^{(k)}_{:, T_j}$ and get $\DD^{(k-1)}_{:, T_j}$ back from $\DD^{(k)}_{:, T_j}$ by reversing~\eqref{eq:Dupdate} as follows 
\eql{\label{eq:Dupdate2}}{
    \DD^{(k-1)}_{:, T_j} = \DD^{(k)}_{:, T_j} - {\VV^U_{T_{j+k-1}, :}}\tp \pr{\sigmoid\prbb{\frac{\QQ^U_{T_j, :}{\KK^U_{T_{j+k-1}, :}}\tp}{\sqrt{\dk}} + \MM^U_{T_j, T_{j+k-1}}}}\tp.  
} 

\noindent\textbf{Diagonal blocks.}
For any $0\leq j \leq N - 1$, 
we first compute $\SSU_{T_{j}, T_j}$ as in~\eqref{eq:SSEdiagonal} and $\AA_{T_{j}, T_j}$ as in~\eqref{eq:ATij}.   
Then, compute the attention weights 
\eql{\label{eq:Pbwddiagonal}}{
    \PP_{T_{j}, T_j} = \exp\pr{\AA_{T_{j}, T_j  } - \mm_{T_{j}}}.  
}
Compute the gradient of attention weights
\eql{\label{eq:dPbwddiagonal}}{
    \de\PP_{T_{j}, T_j} = \de\OO_{T_{j}, :} \VV_{T_j, :}\tp.  
}
Then, update the gradient of $\VV^C_{T_j, :}$ as follows
\eql{\label{eq:dVbwddiagonal}}{
    \de\VV^C_{T_j, :} \leftarrow \de\VV^C_{T_j, :} + \PP_{T_{j}, T_j}\tp \de\OO_{T_{j}, :}.    
} 
Then, compute the gradients of attention weights 
\eql{\label{eq:dSSCbwddiagonal}}{
    \de\SS^C_{T_{j}, T_j} = \PP_{T_{j}, T_j} \odot \pr{\de\PP_{T_{j}, T_j} - \Delta_{T_{j}}} 
}
and 
\eql{\label{eq:dSSUbwddiagonal}}{
    \de\SSU_{T_{j}, T_j} = - \nabla\silu\prb{\SSU_{T_{j}, T_j}} \odot \PP_{T_{j}, T_j} \odot \pr{\de\PP_{T_{j}, T_j} - \Delta_{T_{j}}}.  
}
Then, we update the gradients of $\QQ^C_{T_j, :}$, $\VV^U_{T_j, :}$, $\QQ^U_{T_j, :}$, $\KK^U_{T_j, :}$   through the backward pass of~\eqref{eq:SSEdiagonal} as follows
\eql{\label{eq:dQKVdiagonal}}{
    &\de\QQ^C_{T_j, :} \leftarrow \de\QQ^C_{T_j, :} + \frac{\EE  \VV^U_{T_j, :}}{\sqrt{\dk}}, \\ 
    &\de\VV^U_{T_j, :} \leftarrow \de\VV^U_{T_j, :} + \frac{\EE\tp \QQ^C_{T_j, :}}{\sqrt{\dk}}, \\ 
    &\de\QQ^U_{T_j, :} \leftarrow \de\QQ^U_{T_j, :} + \frac{\FF  \KK^U_{T_j, :}}{\sqrt{\dk}}, \\ 
    &\de\KK^U_{T_j, :} \leftarrow \de\KK^U_{T_j, :} + \frac{\FF\tp \QQ^U_{T_j, :}}{\sqrt{\dk}},  
}
where auxiliary matrices
\eq{
    \EE &= \prbb{\de\SSU_{T_j, T_j} \sigmoid\prbb{\frac{\QQ^U_{T_j, :}{\KK^U_{T_j, :}}\tp}{\sqrt{\dk}} + \MM^U_{T_j, T_j}}} \odot \widetilde{\MM}^C_{T_j, T_j}, \\   
    \FF &= \prbb{\pr{\de\SSU_{T_j, T_j}}\tp \prbb{\frac{\QQ^C_{T_j, :} {\VV^U_{T_j, :}}\tp}{\sqrt{\dk}} \odot \widetilde{\MM}^C_{T_j, T_j}}} \odot \nabla \sigmoid\prbb{\frac{\QQ^U_{T_j, :}{\KK^U_{T_j, :}}\tp}{\sqrt{\dk}} + \MM^U_{T_j, T_j}}.     
}

Then, the pseudo-code of backward pass is illustrated in Algorithm~\ref{alg:traingrad}. 
The backward pass's parallel scheme is naturally the reverse of the forward pass's parallel scheme.  
We remark that when computing the gradients with respect to the block $(T_{j+k}, T_j)$, we update both $\VV^U_{T_{j+k}, :}$ and $\VV^U_{T_{j+k-1}, :}$. Consequently, the block $\VV^U_{T_{j+k-1}, :}$ receives contributions from two sources: $(T_{j+k}, T_j)$ and $(T_{j+k-1}, T_{j-1})$. To prevent overlapping updates, we introduce two auxiliary variables for storing intermediate results in $\de\VV^U$, namely $\de\widehat{\VV}^U$ and $\de\widetilde{\VV}^U$. Specifically, in block $(T_{j+k}, T_j)$, the update to $\de\VV^U_{T_{j+k}, :}$ is accumulated in $\de\widehat{\VV}^U$, while the update to $\de\VV^U_{T_{j+k-1}, :}$ is accumulated in $\de\widetilde{\VV}^U$. After all gradients with respect to off-diagonal blocks have been computed, we obtain the true gradient of $\de\VV^U$ by $\de\VV^U \leftarrow \de\widehat{\VV}^U + \de\widetilde{\VV}^U$. 
The same procedure is applide to $\de\KK^U$.

\section{Efficient Inference with UQ-KV Cache}\label{sec:inferencealg}

\begin{proof}[Proof of~\eqref{eq:Uupdate}.]
    By~\eqref{eq:Ut}, 
    \eq{
        \uu^t_s = \sum_{j=1}^{t}  \sigmoid\prbb{\frac{ \qq^U_s {\kk^U_j}\tp}{\sqrt{\dk}} + [\MM^U_t]_{sj}} \vv^U_j.       
    }
    Thus, for $1\leq s < t$,  
    \eq{
        \uu^t_s - \uu^{t-1}_s = \sigmoid\prbb{\frac{\qq^U_s {\kk^U_{t}}\tp}{\sqrt{\dk}} + [\MM^U_t]_{st}} \vv^U_{t}.  
    }
    Since $[\MM^U_t]_{t, t} = 0$, the last row of $\UU^t$ is all-zeros.  
    This yields~\eqref{eq:Uupdate}.  
\end{proof}

\begin{algorithm}[ht!]
    \caption{Prefilling Algorithm}\label{alg:prefill}
    \KwRequire{$\QQ^U$, $\KK^U$, $\VV^U$, $\QQ^C$, $\KK^C$, $\VV^C \in \Real^{L\x \dk}$ }

    \# Get UQ-KV Cache \;
    Initialize $\UU = \zero^{L \x \dk}$ \;
    \For{$k=0, \dots, N-1$ (in parallel)}{
        Load $\QQ^U_{T_{k}, :}$ from HBM to on-chip SRAM \;
        \For{$j=k, \dots, N-1$ (sequential)}{
            Load $\KK^U_{T_{j}}$, $\VV^U_{T_{j}, :}$, $\UU_{T_{j}, :}$ from HBM to on-chip SRAM \;  
            On chip, compute $\AA^U_{T_k, T_j} = \sigmoid\pr{\frac{\QQ^U_{T_k, :} {\KK^U_{T_j, :}}\tp}{\sqrt{\dk}} + \MM^U_{T_k, T_j}}$ \;
            On chip, $\UU_{T_k, :} \leftarrow \UU_{T_k, :} + \AA^U_{T_k, T_j} \VV^U_{T_j, :}$ \;
            Write $\UU_{T_k, :}$ to HBM
        }
    }
    Save UQ-KV cache $\UU$, $\QQ^U$, $\KK^C$, $\VV^C\in \Real^{L\x \dk}$  \;

    Call Algorithm~\ref{alg:train} to get $\OO$ \;  
    Return $\OO \in \Real^{L\x \dk}$
    
\end{algorithm}

\begin{algorithm}[ht!]
    \caption{Decoding Algorithm (generating the $(t+1)$-th token)}\label{alg:decode}
    \KwRequire{representation $\xx_{t} \in \Real^{1\x \dhidden}$, UQ-KV cache $\UU^{t-1}$, $\QQ^U_{t-1}$, $\KK^C_{t-1}$, $\VV^C_{t-1} \in \Real^{(t-1)\x \dk}$ }

    Compute $\qq^U_{t} = \xx_{t}\WW^U_Q$, $\kk^U_{t} = \xx_{t}\WW^U_K$, $\vv^U_{t} = \xx_{t}\WW^U_V$ \; 
    Compute $\qq^C_{t} = \xx_{t}\WW^C_Q$, $\kk^C_{t} = \xx_{t}\WW^C_K$, $\vv^C_{t} = \xx_{t}\WW^C_V$ \; 
    Update $\UU^t$ as in~\eqref{eq:Uupdate} \;
    Update $\QQ^U_{t} = [\QQ^U_{t-1}; \qq^U_{t}]$ \;
    Update $\KK^C_{t} = [\KK^C_{t-1}; \kk^C_{t}]$ \;
    Update $\VV^C_{t} = [\VV^C_{t-1}; \vv^C_{t}]$ \;
    Compute $\ss^C_t = \frac{\qq^C_{t} {\KK^C_t}\tp}{\sqrt{\dk}} \in \Real^{1\x t}$ as in~\eqref{eq:stC} \;
    Compute $\ss^U_t = {\frac{\qq^C_{t} {\UU^t}\tp}{\sqrt{\dk}} } \in \Real^{1\x t}$ as in~\eqref{eq:stE} \;
    Compute $\pp_t = \softmax\pr{\ss^C_t - \silu\pr{\ss^U_t}} \in \Real^{1\x t}$ as in~\eqref{eq:ptsoftmax1} \;
    Compute $\oo_t = \pp_t \VV^C_t \in \Real^{1\x \dk}$ as in~\eqref{eq:headt} \;
    
    Return $\oo_t \in \Real^{1\x \dk}$ and UQ-KV cache $\UU^t$, $\QQ^U_t$, $\KK^C_t$, $\VV^C_t\in \Real^{t\x \dk}$  
    
\end{algorithm}

\end{document}